\documentclass{alt2019}

\usepackage{amsmath}
\usepackage{wrapfig}

\usepackage{xcolor}
\usepackage{hyperref}
\hypersetup{
  colorlinks   = true,
  urlcolor     = blue,
  linkcolor    = blue,
  citecolor    = blue
}

\definecolor{darkgreen}{rgb}{.0,.5,0}
\definecolor{brown}{rgb}{0.43,0.21,0.1}

\newcommand{\Blue}[1]{\textcolor{blue}{#1}}
\newcommand{\Red}[1]{\textcolor{red}{#1}}
\newcommand{\Green}[1]{\textcolor{darkgreen}{#1}}
\newcommand{\Gray}[1]{\textcolor{gray}{#1}}

\usepackage[mathscr]{euscript}

\newcommand{\scrA}{{\mathscr A}}
\newcommand{\scrL}{{\mathscr L}}
\newcommand{\scrP}{{\mathscr P}}
\newcommand{\scrT}{{\mathscr T}}
\newcommand{\scrU}{{\mathscr U}}
\newcommand{\scrV}{{\mathscr V}}
\newcommand{\scrW}{{\mathscr W}}

\renewcommand{\vec}[1]{\mathbf{\boldsymbol{#1}}}

\newcommand{\gvec}{\vec{g}}
\newcommand{\vvec}{\vec{v}}
\newcommand{\wvec}{\vec{w}}
\newcommand{\wvechat}{\widehat{\wvec}}

\newcommand{\gtil}{\widetilde{g}}

\newcommand{\RR}{\mathbb{R}} 

\DeclareMathOperator*{\argmin}{argmin}

\usepackage{tikz}
\usetikzlibrary{arrows.meta, quotes, automata,positioning, decorations, decorations.pathreplacing, decorations.markings}

\usepackage{float} 
\usepackage{graphicx}
\usepackage[indention=0pt]{caption}
\usepackage[leftcaption]{sidecap}

\usepackage[noend]{algpseudocode}
\usepackage{algorithm}

\usepackage{cancel} 
\usepackage{enumitem} 
\newcommand{\ignore}[1]{} 
\renewcommand{\set}[2][]{#1 \{ #2 #1 \} }
\newcommand{\out}{\text{out}}

\usepackage{amssymb,amsmath}
\usepackage{mathtools} 

\newtheorem{innercustomthm}{Theorem}
\newtheorem{innercustomlemma}{Lemma}
\newtheorem{innercustomproposition}{Proposition}

\newenvironment{customthm}[1]
  {\renewcommand\theinnercustomthm{#1}\innercustomthm}
  {\endinnercustomthm}
\newenvironment{customlemma}[1]
  {\renewcommand\theinnercustomlemma{#1}\innercustomlemma}
  {\endinnercustomlemma}
\newenvironment{customproposition}[1]
  {\renewcommand\theinnercustomproposition{#1}\innercustomproposition}
  {\endinnercustomproposition}

\newcommand{\EXP}{\textsc{EXP3}}
\newcommand{\EXPAG}{\textsc{EXP3-AG}}
\newcommand{\Colon}{\colon\!}

\firstpageno{1}

\title[Online Non-Additive Path Learning]{Online Non-Additive Path Learning \\ 
under Full and Partial Information}

\coltauthor{\Name{Corinna Cortes} \Email{corinna@google.com}\\
 \addr Google Research, New York, NY
 \AND
 \Name{Vitaly Kuznetsov} \Email{vitalyk@google.com}\\
 \addr Google Research, New York, NY
 \AND
 \Name{Mehryar Mohri} \Email{mohri@cims.nyu.edu}\\
 \addr Google Research and Courant Institute, New York, NY
 \AND
 \Name{Holakou Rahmanian}\thanks{Research done in part while the author was visiting Courant Institute and interning at Google Research, NYC as a Ph.D.\ student from UCSC.} \Email{holakou@microsoft.com}\\
 \addr Microsoft Corporation, Redmond, WA
 \AND
 \Name{Manfred K. Warmuth} \Email{manfred@ucsc.edu}\\
 \addr Google Inc., Z\"{u}rich and UC Santa Cruz, CA
 }

\begin{document}

\maketitle

\begin{abstract}

  We study the problem of online path learning with non-additive
  gains, which is a central problem appearing in several applications,
  including ensemble structured prediction. We present new online
  algorithms for path learning with non-additive count-based gains for
  the three settings of full information, semi-bandit and full
  bandit with very favorable regret guarantees. A key component of
  our algorithms is the definition and computation of an intermediate
  context-dependent automaton that enables us to use existing
  algorithms designed for additive gains.  We further apply our
  methods to the important application of ensemble structured
  prediction.  Finally, beyond count-based gains, we give an efficient
  implementation of the EXP3 algorithm for the full bandit setting
  with an arbitrary (non-additive) gain.

\end{abstract}



\section{Introduction}
\label{sec:intro}

One of the core combinatorial online learning problems is that of
learning a minimum loss path in a directed graph.  Examples can be
found in structured prediction problems such as machine translation,
automatic speech recognition, optical character recognition and
computer vision.  In these problems, predictions (or predictors) can
be decomposed into possibly overlapping substructures that may
correspond to words, phonemes, characters, or image patches. They can
be represented in a directed graph where each edge represents a
different substructure.

The number of paths, which serve as \emph{experts}, is typically
exponential in the size of the graph.  Extensive work has been done to
design efficient algorithms when the loss is \emph{additive}, that is
when the loss of the path is the sum of the losses of the edges along
that path.  Several efficient algorithms with favorable guarantees
have been designed both for the full information setting
\citep{TakimotoWarmuth2003, KalaiVempala2005,
  KoolenWarmuthKivinen2010} and different bandit settings
\citep{gyorgy2007line, cesa2012combinatorial} by exploiting the
additivity of the loss.

However, in modern machine learning applications such as machine
translation, speech recognition and computational biology, the loss of
each path is often not additive in the edges along the path.
For instance, in machine translation, the BLEU score similarity
determines the loss. The BLEU score can be closely approximated by the
inner product of the count vectors of the $n$-gram occurrences in two
sequences, where typically $n = 4$ (see Figure~\ref{fig:no-add-ex}).
In computational biology tasks, the losses are determined based on the
inner product of the (discounted) count vectors of occurrences of
$n$-grams with gaps (gappy $n$-grams).
In other applications, such as speech recognition and optical
character recognition, the loss is based on the edit-distance.
Since the performance of the algorithms in these applications is
measured via non-additive loss functions, it is natural to seek
learning algorithms optimizing these losses directly. This motivates
our study of online path learning for non-additive losses.

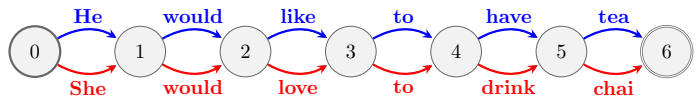
\begin{figure}
\centering
\scalebox{.7}{

\begin{tikzpicture}%
  [>=stealth,
   shorten >=1pt,
   node distance=2cm,
   on grid,
   auto,
   every state/.style={draw=black!60, fill=black!5 }
  ]
  
\tikzset{alg1/.style = {->,> = stealth, very thick, blue}}
\tikzset{alg2/.style = {->,> = stealth, very thick, red}}



\node[state, very thick] (0)                  {0};
\node[state] (1) [right=of 0] {1};
\node[state] (2) [right=of 1] {2};
\node[state] (3) [right=of 2] {3};
\node[state] (4) [right=of 3] {4};
\node[state] (5) [right=of 4] {5};
\node[state, accepting] (6) [right=of 5] {6};

\path[->]
	(0) edge[bend left]	node {\Blue{\textbf{He}}} (1)
		  edge[bend right]	node [below] {\Red{\textbf{She}}} (1)
	(1) edge[bend left]	node {\Blue{\textbf{would}}} (2)
		  edge[bend right]	node [below] {\Red{\textbf{would}}} (2)
	(2) edge[bend left]	node {\Blue{\textbf{like}}} (3)
		  edge[bend right]	node [below] {\Red{\textbf{love}}} (3)
	(3) edge[bend left]	node {\Blue{\textbf{to}}} (4)
		  edge[bend right]	node [below] {\Red{\textbf{to}}} (4)
	(4) edge[bend left]	node {\Blue{\textbf{have}}} (5)
		  edge[bend right]	node [below] {\Red{\textbf{drink}}} (5)
	(5) edge[bend left]	node {\Blue{\textbf{tea}}} (6)
		  edge[bend right]	node [below] {\Red{\textbf{chai}}} (6)
;

\draw[alg1] (0) [bend left] to (1);
\draw[alg1] (1) [bend left] to (2);
\draw[alg1] (2) [bend left] to (3);
\draw[alg1] (3) [bend left] to (4);
\draw[alg1] (4) [bend left] to (5);
\draw[alg1] (5) [bend left] to (6);

\draw[alg2] (0) [bend right] to (1);
\draw[alg2] (1) [bend right] to (2);
\draw[alg2] (2) [bend right] to (3);
\draw[alg2] (3) [bend right] to (4);
\draw[alg2] (4) [bend right] to (5);
\draw[alg2] (5) [bend right] to (6);

\end{tikzpicture}
}
\caption{
  Combining outputs of two different translators (\Blue{blue} and
  \Red{red}).  There are $64$ interleaved translations represented as
  paths.  The BLEU score measures the overlap in $n$-grams between
  sequences. Here, an example of a $4$-gram is
  ``\Blue{like}-\Blue{to}-\Red{drink}-\Blue{tea}''.}
\label{fig:no-add-ex}
\end{figure}

One of the applications of our algorithm is \emph{ensemble structured
  prediction}.  Online learning of ensembles of structured prediction
experts can significantly improve the performance of algorithms in a
number of areas including machine translation, speech recognition,
other language processing areas, optical character recognition, and
computer vision \citep{CortesKuznetsovMohri2014}.  In general,
ensemble structured prediction is motivated by the fact that one
particular expert may be better at predicting one substructure while
some other expert may be more accurate at predicting another
substructure.  Therefore, it is desirable to interleave the
substructure predictions of all experts to obtain the more accurate
prediction.  This application becomes important, particularly in the
bandit setting.  Suppose one wishes to combine the outputs of
different translators as in Figure~\ref{fig:no-add-ex}.  Instead of
comparing oneself to the outputs of the best translator, the
comparator is the best ``interleaved translation'' where each word in
the translation can come from a different translator.  However,
computing the loss or the gain (such as BLEU score) of each path can
be costly and may require the learner to resort to learning from
partial feedback only.

Online path learning with non-additive losses has been previously
studied by \cite{CortesKuznetsovMohriWarmuth2015}. That work focuses
on the full information case providing an efficient implementations of
Expanded Hedge \citep{TakimotoWarmuth2003} and
Follow-the-Perturbed-Leader \citep{KalaiVempala2005} algorithms under
some technical assumptions on the outputs of the experts.

In this paper, we design algorithms for online path learning with
non-additive gains or losses in the full information, as well as in
several bandit settings specified in detail in
Section~\ref{sec:setup}.  In the full information setting, we design
an efficient algorithm that enjoys regret guarantees that are more
favorable than those of \cite{CortesKuznetsovMohriWarmuth2015}, while
not requiring any additional assumption. In the bandit settings, our
algorithms, to the best of our knowledge, are the first efficient
methods for learning with non-additive losses.

The key technical tools used in this work are weighted automata and
transducers \citep{Mohri2009}.  We transform the original path graph
$\scrA$ (e.g.\ Figure~\ref{fig:no-add-ex}) into an intermediate graph
$\scrA'$.  The paths in $\scrA$ are mapped to the paths in $\scrA'$,
but now the losses in $\scrA'$ are additive along the paths.
Remarkably, the size of $\scrA'$ does not depend on the size of the
alphabet (word vocabulary in translation tasks) from which the output
labels of edges are drawn.  The construction of $\scrA'$ is highly
non-trivial and is our primary contribution.  This alternative graph
$\scrA'$, in which the losses are additive, enables us to extend many
well-known algorithms in the literature to the path learning problem.

The paper is organized as follows.  We introduce the path learning
setup in Section~\ref{sec:setup}.  In Section~\ref{sec:count-gains},
we explore the wide family of non-additive count-based gains and
introduce the alternative graph $\scrA'$ using automata and
transducers tools.  We present our algorithms in
Section~\ref{sec:algs} for the full information, semi- and full bandit
settings for the count-based gains.  Next, we extend our results to
\emph{gappy} count-based gains in Section~\ref{sec:gappy-gains}.  
The application of our method to the ensemble structured prediction is
detailed in Appendix~\ref{app:ensemble}.  In Appendix~\ref{sec:exp3},
we go beyond count-based gains and consider \emph{arbitrary}
(non-additive) gains.  Even with no assumption about the structure of
the gains, we can efficiently implement the EXP3 algorithm in the full
bandit setting. Naturally, the regret bounds for this algorithm are
weaker, however, since no special structure of the gains can be
exploited in the absence of any assumption.

\section{Basic Notation and Setup}
\label{sec:setup}

We describe our path learning setup in terms of finite automata.  Let
$\scrA$ denote a fixed acyclic finite automaton.  We call $\scrA$ the
\emph{expert automaton}. $\scrA$ admits a single \emph{initial state}
and one or several \emph{final states} which are indicated by bold and
double circles, respectively, see Figure~\ref{fig:experts}(a). Each
transition of $\scrA$ is labeled with a unique \emph{name}.  Denote
the set of all transition names by $E$. An automaton with a single
initial state is \emph{deterministic} if no two outgoing transitions
from a given state admit the same name. Thus, our automaton $\scrA$ is
deterministic by construction since the transition names are unique.
An \emph{accepting path} is a sequence of transitions from the initial
state to a final state.  The expert automaton $\scrA$ can be viewed as
an indicator function over strings in $E^*$ such that
$\scrA( \pi )\!=\!1 $ iff $\pi$ is an accepting path. Each accepting
path serves as an expert and we equivalently refer to it as a
\emph{path expert}. The set of all path experts is denoted by $\scrP$.
At each round $t = 1, \ldots, T$, each transition $e \in E$ outputs a
\emph{symbol} from a finite non-empty alphabet $\Sigma$, denoted by
$\text{out}_t(e) \in \Sigma$.  The \emph{prediction} of each path
expert $\pi \in E^*$ at round $t$ is the sequence of output symbols
along its transitions at that round and is denoted by
$\text{out}_t(\pi) \in \Sigma^*$.  We also denote by
$\text{out}_t(\scrA)$ the automaton with the same topology as $\scrA$
where each transition $e$ is labeled with $\text{out}_t(e)$, see
Figure~\ref{fig:experts}(b).
At each round $t$, a \emph{target sequence} $y_t \in \Sigma^*$ is
presented to the learner.  The \emph{gain/loss} of each path expert
$\pi$ is $\scrU(\text{out}_t(\pi), y_t)$ where
$\scrU\colon \Sigma^* \times \Sigma^* \longrightarrow \mathbb{R}_{\geq
  0}$.  Our focus is the $\scrU$ functions that are not necessarily
additive along the transitions in $\scrA$.  For example, $\scrU$ can
be either a \emph{distance function} (e.g. edit-distance) or a
\emph{similarity function} (e.g. $n$-gram gain with $n\geq 2$).

\input{a2.fig}
\input{a3.fig}

We consider standard online learning scenarios of prediction with path
experts.  At each round $t\!\in\![T]$, the \emph{learner} picks a path
expert $\pi_t$ and predicts with its prediction $\text{out}_t(\pi_t)$.
The learner receives the gain of $\scrU(\text{out}_t(\pi_t),y_t)$.
Depending on the setting, the \emph{adversary} may reveal some
information about $y_t$ and the output symbols of the transitions (see
Figure~\ref{fig:games}).  In the \emph{full information} setting,
$y_t$ and $\text{out}_t(e)$ are revealed to the learner for every
transition $e$ in $\scrA$.  In the \emph{semi-bandit} setting, the
adversary reveals $y_t$ and $\text{out}_t(e)$ for every transition $e$
along $\pi_t$.  In \emph{full bandit} setting,
$\scrU(\text{out}_t(\pi_t),y_t)$ is the only information that is
revealed to the learner.  The goal of the learner is to minimize the
\emph{regret} which is defined as the cumulative gain of the best path
expert chosen in hindsight minus the cumulative expected gain of the
learner.

\section{Count-Based Gains}
\label{sec:count-gains}

Many of the most commonly used non-additive gains in applications
belong to the broad family of \emph{count-based gains}, which are
defined in terms of the number of occurrences of a fixed set of
patterns, $\theta_1, \theta_2, \ldots, \theta_p$, in the sequence
output by a path expert. These patterns may be $n$-grams, that is
sequences of $n$ consecutive symbols, as in a common approximation of
the BLEU score in machine translation, a set of relevant subsequences
of variable-length in computational biology, or patterns described by
complex regular expressions in pronunciation
modeling. 

For any sequence $y \in \Sigma^*$, let $\Theta(y) \in \RR^p$ denote
the vector whose $k$th component is the number of occurrences
of $\theta_k$ in $y$, $k \in [p]$.\footnote{This can be extended to
  the case of weighted occurrences where more emphasis is assigned to
  some patterns $\theta_k$ whose occurrences are then multiplied by a
  factor $\alpha_k > 1$, and less emphasis to others.}  The
count-based gain function $\scrU$ at round $t$ for a path expert $\pi$
in $\scrA$ given the target sequence $y_t$ is then defined as a dot
product:
\begin{equation}
\label{eq:gain}
\scrU(\text{out}_t(\pi), y_t) := \Theta(\text{out}_t(\pi)) \cdot \Theta(y_t) \geq 0.
\end{equation}
Such gains are not additive along the transitions and 
the standard online path learning algorithms for additive gains cannot
be applied.
Consider, for example, the special case of $4$-gram-based gains in Figure~\ref{fig:no-add-ex}.
These gains cannot be expressed additively if the target sequence is, for instance,
``He would like to eat cake'' (see Appendix~\ref{app:non-additive-example}).
%
%
The challenge of learning with non-additive gains is even more
apparent in the case of \emph{gappy} count-based gains which allow for
gaps of varying length in the patterns of interest.  We defer the
study of gappy-count based gains to Section~\ref{sec:gappy-gains}.

How can we design algorithms for online path learning with such
non-additive gains?  Can we design algorithms with favorable regret
guarantees for all three settings of full information, semi- and full
bandit?  The key idea behind our solution is to design a new automaton
$\scrA'$ whose paths can be identified with those of $\scrA$ and,
crucially, whose gains are additive.  We will construct $\scrA'$ by
defining a set of \emph{context-dependent rewrite rules}, which can be
compiled into a finite-state transducer $T_{\scrA}$ defined below. The
\emph{context-dependent automaton $\scrA'$} can then be obtained by
composition of the transducer $T_{\scrA}$ with $\scrA$.  In addition
to playing a key role in the design of our algorithms
(Section~\ref{sec:algs}), $\scrA'$ provides a compact representation
of the gains since its size is substantially less than the dimension
$p$ (number of patterns).

\vspace{-.1cm}

\subsection{Context-Dependent Rewrite Rules}
\label{subsec:cdrr}

We will use \emph{context-dependent rewrite rules} to map $\scrA$ to
the new representation $\scrA'$. These are rules that admit the
following general form:
\begin{equation*}
\phi \rightarrow \psi \slash \lambda \underline{~~~~~} \rho,
\end{equation*}
where $\phi$, $\psi$, $\lambda$, and $\rho$ are regular expressions
over the alphabet of the rules.
These rules must be interpreted as
follows: $\phi$ is to be replaced by $\psi$ whenever it is preceded by
$\lambda$ and followed by $\rho$. Thus, $\lambda$ and $\rho$ represent
the left and right contexts of application of the rules. Several types
of rules can be considered depending on their being obligatory or
optional, and on their direction of application, from left to right,
right to left or simultaneous application \citep{KaplanKay1994}.
We will be only considering rules with simultaneous applications.
\input{ta.fig}
Such context-dependent rules can be efficiently compiled into a
\emph{finite-state transducer} (FST), under the technical condition
that they do not rewrite their non-contextual part
\citep{MohriSproat1996,KaplanKay1994}.\footnote{Additionally, the
  rules can be augmented with weights, which can help us cover the
  case of weighted count-based gains, in which case the result of the
  compilation is a weighted transducer \citep{MohriSproat1996}. Our
  algorithms and theory can be extended to that case.}  An FST $\scrT$
over an input alphabet $\Sigma$ and output alphabet $\Sigma'$ defines
an indicator function over the pairs of strings in
$\Sigma^* \times \Sigma'^*$.  Given $x \in \Sigma^*$ and
$y \in \Sigma'^*$, we have $\scrT(x, y) = 1$ if there exists a path
from an initial state to a final state with input label $x$ and output
label $y$, and $\scrT(x, y) = 0$ otherwise.

To define our rules, we first introduce the alphabet $E'$ 
as the set of transition names for the target automaton
$\scrA'$. These  capture all possible contexts of length
$r$, where $r$ is the length of pattern $\theta_k$:
\begin{equation*}
E' = \set[\Big]{ \#e_1 \cdots e_r  \mid e_1 \cdots e_r \text{ is a path segment of
    length $r$ in
    $\scrA$}, r \in  \set[\big]{ |\theta_1|, \ldots ,|\theta_p| } },
\end{equation*}
where the `\#' symbol ``glues'' $e_1, \ldots, e_r \in E$ together and forms one single symbol in $E'$.
We will have one context-dependent rule of the following form for each
element $\#e_1 \cdots e_r \in E'$:
\begin{equation}
\label{eq:rules}
e_1 \cdots e_r \rightarrow \#e_1 \cdots e_r \slash \epsilon
\underline{~~~~~} \epsilon .
\end{equation}
Thus, in our case, the left- and right-contexts are the empty strings%
\footnote{
Context-dependent rewrite rules are powerful tools for identifying different patterns
using their left- and right-contexts. For our application of count-based gains, however, 
identifying these patterns are independent of their context and 
we do not need to fully exploit the strength of these rewrite rules.
}%
,
meaning that the rules can apply (simultaneously) at every
position. In the special case where the patterns $\theta_k$ are the
set of $n$-grams, then $r$ is fixed and equal to $n$.
Figure~\ref{fig:ex_T_A} shows the result of the rule compilation in
that case for $n = 2$. This transducer inserts $\#e_1e_3$ whenever
$e_1$ and $e_3$ are found consecutively and otherwise outputs the empty string.
We will denote the resulting FST by $T_\scrA$.


\subsection{Context-Dependent Automaton $\scrA'$}
\label{subsec:additiveA}

To construct the context-dependent automaton $\scrA'$, we will use the
\emph{composition} operation.
The composition of $\scrA$ and $T_\scrA$ is an FST denoted by
$\scrA \circ T_\scrA$ and defined as the following product of 
two $0/1$ outcomes for all inputs:
\begin{equation*}
\forall x \in E^*, \; \forall y \in E'^*\Colon \quad (\scrA \circ T_\scrA)(x, y) := \scrA(x) \cdot T_\scrA(x, y).
\end{equation*}
There is an efficient algorithm for the composition of FSTs and
automata \citep{PereiraRiley1997,MohriPereiraRiley1996,Mohri2009},
whose worst-case complexity is in $O(|\scrA| \, |T_\scrA|)$.  The
automaton $\scrA'$ is obtained from the FST $(\scrA \circ T_\scrA)$ by
\emph{projection}, that is by simply omitting the input label of each
transition and keeping only the output label. Thus if we denote by
$\Pi$ the projection operator, then $\scrA'$ is defined as
$\;\scrA' = \Pi(\scrA \circ T_\scrA).$

Observe that $\scrA'$ admits a fixed topology (states and transitions)
at any round $t \in [T]$. It can be constructed in a pre-processing
stage using the FST operations of composition and
projection. Additional FST operations such as $\epsilon$-removal and
minimization can help further optimize the automaton obtained after
projection \citep{Mohri2009}.
\ignore{
Notice that the transducer $T_\scrA$ of Figure~\ref{fig:ex_T_A}(b) is
\emph{deterministic}, that is no two transitions leaving a state admit
the same input label. More generally, in the case where the patterns
$\theta_k$ are $n$-grams, $T_\scrA$ can be constructed to be
deterministic without increasing its size: the number of transitions
coincides with the number of elements in $E'$. 

Moreover, the result of the rule compilation can be determinized using
transducer determinization \citep{Mohri2009}.  If determinized,
$T_\scrA$ assigns a unique output to a given path expert in $\scrA$.}
Proposition~\ref{prop:1-1}, proven in Appendix~\ref{app:proofs},
ensures that for every accepting path $\pi$ in $\scrA$, there is a
unique corresponding accepting path in $\scrA'$.  Figure~\ref{fig:aa'}
shows the automata $\scrA$ and $\scrA'$ in a simple case and how a
path $\pi$ in $\scrA$ is mapped to another path $\pi'$ in $\scrA'$.

\begin{proposition}
\label{prop:1-1}
Let $\scrA$ be an expert automaton and let $T_\scrA$ be a
deterministic transducer representing the rewrite
rules~\eqref{eq:rules}.  Then, for each accepting path $\pi$ in
$\scrA$, there exists a unique corresponding accepting path $\pi'$ in
$\scrA' = \Pi ( \scrA \circ T_\scrA )$.
\end{proposition}
\input{aap.fig}
The size of the context-dependent automaton $\scrA'$ depends on the
expert automaton $\scrA$ and the lengths of the patterns.
Notice that, crucially, its size is independent of the size of the
alphabet $\Sigma$.  Appendix~\ref{app:ensemble} analyzes more
specifically the size of $\scrA'$ in the important application of
ensemble structure prediction with $n$-gram gains.

At any round $t\!\in\![T]$ and for any $\#e_1 \cdots e_r\!\in\!E'$,
let $\out_t(\#e_1 \cdots e_r)$ denote the sequence
$\out_t(e_1) \cdots \out_t(e_r)$, that is the sequence obtained by
concatenating the outputs of $e_1, \ldots, e_r$.  Let $\out_t(\scrA')$
be the automaton with the same topology as $\scrA'$ where each label
$e'\!\in\!E'$ is replaced by $\out_t(e')$.
Once $y_t$ is known, the representation $\Theta(y_t)$ can be found, and
consequently, the additive contribution of each transition of
$\scrA'$ can be computed.
The following theorem, which is proved in Appendix~\ref{app:proofs},
shows the additivity of the gains in $\scrA'$.
See Figure~\ref{fig:outaa'} for an example.

\begin{theorem}
\label{thm:additivity}
At any round $t \in [T]$, define the gain $g_{e'\!,t}$ of the
transition $e' \in E'$ in $\scrA'$ by
$g_{e'\!,t} := \left[ \Theta(y_t) \right]_k$ if
$\out_t(e') = \theta_k$ for some $k \in [p]$ and $g_{e'\!,t}:=0$ if no
such $k$ exists.  Then, the gain of each path $\pi$ in $\scrA$ at
trial $t$ can be expressed as an additive gain of the corresponding
unique path $\pi'$ in $\scrA'$:
\begin{equation*}
\forall t \in [T], \;
\forall \pi \in \scrP \Colon \quad
\scrU(\out_t(\pi), y_t)
= \sum_{e' \in \pi'} g_{e'\!,t} \;.
\end{equation*}
\end{theorem}

\ignore{
As an example, consider the automaton $\scrA$ and its associated
context-dependent automaton $\scrA'$ shown in Figure~\ref{fig:aa'},
with bigram gains and $\Sigma = \{ a, b \}$.  Here, the patterns are
$(\theta_1, \theta_2, \theta_3, \theta_4) = (aa, ab, ba, bb)$.  Let
the target sequence at trial $t$ be $y_t = aba$.  Thus
$\Theta(y_t)=[0, 1, 1, 0]^T$.  The automata $\out_t(\scrA)$ and
$\out_t(\scrA')$ are given in Figure~\ref{fig:outaa'}.
}

\input{outaap.fig}

\vspace{-.4cm}

\section{Algorithms}
\label{sec:algs}

In this section, we present algorithms and associated regret
guarantees for online path learning with non-additive count-based
gains in the full information, semi-bandit and full bandit
settings. The key component of our algorithms is the context-dependent
automaton $\scrA'$.
In what follows, we denote the length of the longest path in $\scrA'$
by $K$, an upper-bound on the gain of each transition in $\scrA'$ by
$B$, the number of path experts by $N$, and the number of transitions
and states in $\scrA'$ by $M$ and $Q$, respectively. We note that $K$
is at most the length of the longest path in $\scrA$ since each
transition in $\scrA'$ admits a unique label.
\paragraph{Remark.} 
The number of accepting paths in $\scrA'$ is often equal to but
sometimes less than the number of accepting paths in $\scrA$.  In some
degenerate cases, several paths $\pi_1, \ldots, \pi_k$ in $\scrA$ may
correspond to one single path%
\footnote{
For example, in the case of $n$-gram gains, 
all the paths in $\scrA$ with a length less than $n$
correspond to path with empty output in $\scrA'$
and will always have a gain of zero.
}
$\pi'$ in $\scrA'$.  This implies that
$\pi_1, \ldots, \pi_k$ in $\scrA$ will always consistently have the
same gains in every round and that is the additive gain of $\pi'$ in
$\scrA'$.  Thus, if $\pi'$ is predicted by the algorithm in $\scrA'$,
any of the paths $\pi_1, \ldots, \pi_k$ can be equivalently used for
prediction in the original expert automaton $\scrA$.

\subsection{Full Information: Context-dependent Component Hedge
  Algorithm}

\cite{KoolenWarmuthKivinen2010} gave an algorithm for online path
learning with non-negative additive losses in the full information
setting, the Component Hedge (CH) algorithm.  For count-based losses,
\cite{CortesKuznetsovMohriWarmuth2015} provided an efficient Rational
Randomized Weighted Majority (RRWM) algorithm. This algorithm requires
the use of determinization \citep{Mohri2009} which is only shown to
have polynomial computational complexity under some additional
technical assumptions on the outputs of the path experts. In this
section, we present an extension of CH, the \emph{Context-dependent
  Component Hedge} (CDCH), for the online path learning problem with
non-additive count-based gains. CDCH admits more favorable regret
guarantees than RRWM and can be efficiently implemented without any
additional assumptions.

Our CDCH algorithm requires a modification of $\scrA'$ such that all
paths admit an equal number $K$ of transitions (same as the longest
path).  This modification can be done by adding at most $(K-2)(Q-2)+1$
states and zero-gain transitions \citep{gyorgy2007line}. Abusing the
notation, we will denote this new automaton by $\scrA'$ in this
subsection.
At each iteration $t$, CDCH maintains a weight vector $\wvec_t$
in the unit-flow polytope $P$ over $\scrA'$, which is a set
of vectors $\wvec \in \RR^M$ satisfying the following conditions:
(1) the weights of the outgoing transitions from the initial state sum up to one, and
(2) for every non-final state, the sum of the weights of incoming and outgoing transitions are equal.
For each $t \in \{1, \ldots, T\}$, we observe the gain of each
transition $g_{t, e'}$, and define the loss of that transition as
$\ell_{e'} = B - g_{t, e'}$.  After observing the loss of each
transition $e'$ in $\scrA'$, CDCH updates each component of $\wvec$ as
$\widehat{w}(e') \leftarrow w_t(e') \, \exp( - \eta \, \ell_{t,e'})$
(where $\eta$ is a specified learning rate), and sets $\wvec_{t+1}$ to
the relative entropy projection of the updated $\widehat{\wvec}$ back
to the unit-flow polytope, i.e.\
$\wvec_{t+1} = \argmin_{\wvec \in P} \sum_{e' \in E'} w(e') \ln
\frac{w(e')}{\widehat{w}(e')} + \widehat{w}(e') - w(e') $.

CDCH predicts by decomposing $\wvec_t$ into a convex combination of at
most $|E'|$ paths in $\scrA'$ and then sampling a single path
according to this mixture as described below.  Recall that each path in
$\scrA'$  identifies a path in $\scrA$ which can be recovered
in time $K$. Therefore, the inference step of the CDCH algorithm takes
at most time polynomial in $|E'|$ steps.  To determine a
decomposition, we find a path from the initial state to a final state
with non-zero weights on all transitions, remove the largest weight on
that path from each transition on that path and use it as a mixture
weight for that path. The algorithm proceeds in this way until the
outflow from initial state is zero.  The following theorem from
\citep{KoolenWarmuthKivinen2010} gives a regret guarantee for the CDCH
algorithm.

\begin{theorem}
\label{th:nch}
With proper tuning of the learning rate $\eta$, the regret of CDCH is
bounded as below:
\[
  \forall \; \pi^* \in \scrP \Colon \quad \sum_{t=1}^n
  \scrU(\out_t(\pi^*), y_t) - \scrU(\out_t(\pi_t), y_t) \leq \sqrt{2
    \, T \, B^2 K^2 \, \log (K \, M)} + B \, K \, \log (K M).
\]
\end{theorem}
The regret bounds of Theorem~\ref{th:nch} are in terms of the
count-based gain $\scrU(\cdot, \cdot)$.
\cite{CortesKuznetsovMohriWarmuth2015} gave regret guarantees for the
RRWM algorithm with count-based losses defined by
$- \log \scrU (\cdot, \cdot)$.  In Appendix \ref{app:log-reg}, we show
that the regret associated with $-\log \scrU$ is upper-bounded by the
regret bound associated with $\scrU$.  Observe that, even with this
approximation, the regret guarantees that we provide for CDCH are
tighter by a factor of $K$. In addition, our algorithm does not
require additional assumptions for an efficient implementation
compared to the RRWM algorithm of
\cite{CortesKuznetsovMohriWarmuth2015}.

\subsection{Semi-Bandit: Context-dependent Semi-Bandit Algorithm}

\cite{gyorgy2007line} gave an efficient algorithm for online path
learning with additive losses in the semi-bandit setting. In this
section, we present a \emph{Context-dependent Semi-Bandit} (CDSB)
algorithm extending that work to solving the problem of online path
learning with count-based gains in a semi-bandit setting.  To the best
of our knowledge, this is the first efficient algorithm with favorable
regret bounds for this problem.

As with the algorithm of \cite{gyorgy2007line}, CDSB makes use of a
set $C$ of \emph{covering paths} with the property that, for each
$e' \in E'$, there is an accepting path $\pi'$ in $C$ such that $e'$
belongs to $\pi'$.  At each round $t$, CDSB keeps track of a
distribution $p_t$ over all $N$ path experts by maintaining a weight
$w_t(e')$ on each transition $e'$ in $\scrA'$ such that the weights of
outgoing transitions for each state sum up to $1$ and
$p_t(\pi') = \prod_{e' \in \pi'} w_t(e')$, for all accepting paths
$\pi'$ in $\scrA'$. Therefore, we can sample a path $\pi'$ from $p_t$
in at most $K$ steps by selecting a random transition at each state
according to the distribution defined by $\wvec_t$. To make a
prediction, we sample a path in $\scrA'$ according to a mixture
distribution $(1-\gamma)p_t + \gamma \mu$, where $\mu$ is a uniform
distribution over paths in $C$.  We select $p_t$ with probability
$1-\gamma$ or $\mu$ with probability $\gamma$ and sample a random path
$\pi'$ from the randomly chosen distribution.

Once a path $\pi'_t$ in $\scrA'$ is sampled, we observe the gain of
each transition $e'$ of $\pi'_t$, denoted by $g_{t, e'}$.  CDSB sets
$\widehat{w}_t(e')\!=\!w_{t}(e') \exp(\eta \widetilde{g}_{t, e'})$,
where $\widetilde{g}_{t, e'}\!=\!(g_{t, e'} + \beta) / q_{t, e'}$ if
$e'\!\in\!\pi'_t$ and $\widetilde{g}_{t, e'} = \beta / q_{t, e'}$
otherwise.  Here, $\eta, \beta, \gamma\!>\!0$ 
are parameters of the algorithm and $q_{t, e'}$ is the flow through
$e'$ in $\scrA'$, which can be computed using a standard
shortest-distance algorithm over the probability
semiring~\citep{Mohri2009}.  The updated distribution is
$p_{t+1}(\pi') \propto \prod_{e' \in \pi'} \widehat{w}_t(e')$.  Next,
the \emph{weight pushing} algorithm \citep{Mohri1997} is applied  (see
Appendix~\ref{app:fst}), which results in new transition weights
$\wvec_{t+1}$ such that the total outflow out of each state is again
one and the updated probabilities are
$p_{t+1}(\pi') = \prod_{e' \in \pi'} w_{t+1}(e')$, thereby
facilitating sampling.  The computational complexity of each of the
steps above is polynomial in the size of $\scrA'$. The following
theorem from \cite{gyorgy2007line} provides a regret guarantee for
CDSB algorithm.

\begin{theorem}
\label{thm:cdsb}
Let $C$ denote the set of ``covering paths'' in $\scrA'$.  For any
$\delta \in (0, 1)$, with proper tuning of the parameters $\eta$,
$\beta$, and $\gamma$, the regret of the CDSB algorithm can be bounded
as follows with probability $1-\delta$:
\[
\forall \; \pi^* \in \scrP \Colon \quad 
\sum_{t=1}^n  \scrU(\text{out}_t(\pi^*), y_t) - \scrU(\text{out}_t(\pi_t), y_t)  
\leq  2 \, B \, \sqrt{TK} \Big( \sqrt{4 K |C| \ln N} + \sqrt{M \ln \tfrac{M}{\delta}} \Big).
\]
\end{theorem}

\subsection{Full Bandit: Context-dependent ComBand Algorithm}

Here, we present an algorithm for online path learning with
count-based gains in the full bandit setting.
\cite{cesa2012combinatorial} gave an algorithm for online path
learning with additive gains, \textsc{ComBand}.  Our generalization,
called \emph{Context-dependent ComBand} (CDCB), is the first efficient
algorithm with favorable regret guarantees for learning with
count-based gains in this setting.  For the full bandit setting with
\emph{arbitrary} gains, we develop an efficient execution of \EXP,
called \EXPAG, in Appendix~\ref{sec:exp3}.

As with CDSB, CDCB maintains a distribution $p_t$ over all $N$ path
experts using weights $\wvec_t$ on the transitions such that the
outflow of each state is one and the probability of each path experts
is the product of the weights of the transitions along that path.  To
make a prediction, we sample a path in $\scrA'$ according to a mixture
distribution $q_t = (1-\gamma) p_t + \gamma \mu$, where $\mu$ is a
uniform distribution over the paths in $\scrA'$.  Note that this
sampling can be efficiently implemented as follows. As a
pre-processing step, define $\mu$ using a separate set of weights
$\wvec^{(\mu)}$ over the transitions of $\scrA'$ in the same form.
Set all the weights $\wvec^{(\mu)}$ to one and apply the
weight-pushing algorithm to obtain a uniform distribution over the
path experts.  Next, we select $p_t$ with probability $1 - \gamma$ or
$\mu$ with probability $\gamma$ and sample a random path $\pi'$ from
the randomly chosen distribution.

After observing the scalar gain $g_{\pi'}$ of the chosen path, CDCB
computes a surrogate gain vector for all transitions in $\scrA'$ via
$ \widetilde{\gvec}_{t} \!=\!g_{\pi'} P \vvec_{\pi'}$, where $P$ is
the pseudo-inverse of $\mathbb{E}[\vvec_{\pi'} \vvec_{\pi'}^T]$ and
$\vvec_{\pi'}\!\in\!\{0,1\}^{M}$ is a bit representation of the path
$\pi'$.  As for CDSB, we set
$\widehat{w}(e') = w_{t}(e') \exp(-\eta \widetilde{g}_{t, e'})$ and
update $\scrA'$ via weighted-pushing to compute $\wvec_{t+1}$. We
obtain the following regret guarantees from
\cite{cesa2012combinatorial} for CDCB:

\begin{theorem}
\label{thm:cdcb}
Let $\lambda_\text{min}$ denote the smallest non-zero eigenvalue of
$\mathbb{E}[\vvec_{\pi'} \vvec_{\pi'}^T]$
where $\vvec_{\pi'} \in \{0,1\}^{M}$ is the bit representation of the path $\pi'$
which is distributed according to the uniform distribution $\mu$.
With proper tuning of the parameters $\eta$ and $\gamma$, the regret
of CDCB can be bounded as follows:
$$
\forall \; \pi^* \in \scrP \Colon \quad 
\sum_{t=1}^n  \scrU(\text{out}_t(\pi^*), y_t) - \scrU(\text{out}_t(\pi_t), y_t)  
\leq 2 \, B \, \sqrt{\bigg( \frac{2K}{M \lambda_\text{min}} + 1 \bigg) T M \ln N}.
$$
\end{theorem}


\section{Extension to Gappy Count-Based Gains}
\label{sec:gappy-gains}

Here, we generalize the results of Section~\ref{sec:count-gains} to a
broader family of non-additive gains called \emph{gappy count-based
  gains}: the gain of each path depends on the \emph{discounted}
counts of \emph{gappy} occurrences of a fixed set of patterns
$\theta_1, \ldots, \theta_p$ in the sequence output by that path.  In
a gappy occurrence, there can be ``gaps'' between symbols of the
pattern.  The count of a gappy occurrence is discounted
multiplicatively by $\gamma^k$ where $\gamma \in [0,1]$ is a fixed
\emph{discount rate} and $k$ is the total length of gaps.  For
example, the gappy occurrences of the pattern $\theta = \Blue{aab}$ in
a sequence $y = babbaabaa$ with discount rate $\gamma$ are

\begin{itemize}[leftmargin=.1cm,itemindent=.35cm]
\setlength\itemsep{.1cm}

\item ${b\,a\,b\,b\,}\Blue{\underline{a\,a\,b}}\,{a\,a}$, length of
  gap = $0$, discount factor = $1$;

\item
  ${b\,}\Blue{\underline{a}}\,\Gray{b}\,\Gray{b}\,\Blue{\underline{a}}\,\Gray{a}\,\Blue{\underline{b}}\,{a\,a}$,
  length of gap = $3$, discount factor = $\gamma^3$;

\item
  ${b}\,\Blue{\underline{a}}\,\Gray{b}\,\Gray{b}\,\Gray{a}\,\Blue{\underline{a\,b}}\,{a\,a}$,
  length of gap = $3$, discount factor = $\gamma^3$,

\end{itemize}
which makes the total discounted count of gappy occurrences of
$\theta$ in $y$ to be $1 + 2 \cdot \gamma^3$.  Each sequence of
symbols $y \in \Sigma^*$ can be represented as a discounted count
vector $\Theta(y) \in \RR^p$ of gappy occurrences of the patterns
whose $i$th component is ``the discounted number of gappy occurrences
of $\theta_i$ in $y$''.  The gain function $\scrU$ is defined in the
same way as in Equation~\eqref{eq:gain}.\footnote{The regular
  count-based gain can be recovered by setting $\gamma=0$.} A typical
instance of such gains is gappy $n$-gram gains where the patterns are
all $|\Sigma|^n$-many $n$-grams.

The key to extending our results in Section~\ref{sec:count-gains}
to gappy $n$-grams is an appropriate definition of the alphabet $E'$, 
the rewrite rules, and a new context-dependent automaton $\scrA'$. 
Once $\scrA'$ is constructed, the algorithms and regret guarantees
presented in Section~\ref{sec:algs} can be extended to gappy count-based gains.
To the best of our knowledge, this provides the first efficient
online algorithms with favorable regret guarantees
for gappy count-based gains in full information, semi-bandit and
full bandit settings.

\paragraph{Context-Dependent Rewrite Rules.}
We extend the definition of $E'$ so that it also encodes the total
length $k$ of the gaps:
$E' = \Big\{ (\#e_1 \cdots e_r)_k \mid e_1 \cdots e_r \in E, \; r \in
\{ |\theta_1|, \ldots ,|\theta_p| \}, \newline \; k \in \mathbb{Z}, \;
k\geq0 \Big\}$.  Note that the discount factor in gappy occurrences
does not depend on the position of the gaps.  Exploiting this fact,
for each pattern of length $n$ and total gap length $k$, we reduce the
number of output symbols by a factor of ${k+n-2}\choose{k}$ by
encoding the number of gaps as opposed to the position of the gaps.

We extend the rewrite rules in order to incorporate the gappy
occurrences.  Given $e' = (\# e_{i_1}e_{i_2}\ldots e_{i_n})_k$, for
all path segments $e_{j_1}e_{j_2}\ldots e_{j_{n+k}}$ of length $n+k$
in $\scrA$ where $\{i_s\}_{s=1}^n$ is a subsequence of
$\{j_r\}_{r=1}^{n+k}$ with ${i_1}={j_1}$ and ${i_n}={j_{n+k}}$, we
introduce the rule:
\[
e_{j_1}e_{j_2}\ldots e_{j_{n+k}} 
\longrightarrow (\# e_{i_1}e_{i_2}\ldots e_{i_n})_k
/ \epsilon {\underline{\hspace{.5cm}}} \epsilon.
\]
As with the non-gappy case in Section~\ref{sec:count-gains}, the
simultaneous application of all these rewrite rules can be efficiently
compiled into a FST $T_\scrA$.  The context-dependent transducer
$T_\scrA$ maps any sequence of transition names in $E$ into a sequence
of corresponding gappy occurrences.  The example below shows how
$T_\scrA$ outputs the gappy trigrams given a path segment of length
$5$ as input:
\begin{align*}
e_1, e_2, e_3, e_4, e_5 \xrightarrow{T_\scrA}
& (\#e_1e_2e_3)_0, (\#e_2e_3e_4)_0, (\#e_3e_4e_5)_0, 
 (\#e_1e_2e_4)_1, (\#e_1e_3e_4)_1,  (\#e_2e_3e_5)_1, \\
& (\#e_2e_4e_5)_1, (\#e_1e_2e_5)_2, (\#e_1e_4e_5)_2, (\#e_1e_3e_5)_2.
\end{align*}

\paragraph{Context-Dependent Automaton $\scrA'$.}
As in Section \ref{subsec:additiveA}, we construct the
context-dependent automaton as
$\scrA'\!:=\!\Pi ( \scrA\!\circ\!T_\scrA )$, which admits a fixed
topology through trials.  The rewrite rules are constructed in a way
such that different paths in $\scrA$ are rewritten differently.
Therefore, $T_\scrA$ assigns a unique output to a given path expert in
$\scrA$.  Proposition~\ref{prop:1-1} ensures that for every accepting
path $\pi$ in $\scrA$, there is a unique corresponding accepting path
in $\scrA'$.

For any round $t\!\in\![T]$ and any
$e'\!=\!(\# e_{i_1} e_{i_2} \cdots e_{i_n})_k$, define
$\text{out}_t( e'
):=\text{out}_t(e_{i_1})\ldots\text{out}_t(e_{i_n})$.  Let
$\text{out}_t(\scrA')$ be the automaton with the same topology as
$\scrA'$ where each label $e'\!\in\!E'$ is replaced by
$\text{out}_t(e')$.  Given $y_t$, the representation $\Theta(y_t)$ can
be found, and consequently, the additive contribution of each
transition of $\scrA'$.  Again, we show the additivity of the gain in
$\scrA'$ (see Appendix~\ref{app:proofs} for the proof).

\begin{theorem}
\label{thm:gap-additivity}
Given the trial $t$ and discount rate $\gamma \in [0,1]$, for each transition $e' \in E'$ in $\scrA'$, define the gain $g_{e'\!,t} := \gamma^k \left[ \Theta(y_t) \right]_i$ 
if $\text{out}_t(e') = (\theta_i)_k$ for some $i$ and $k$ 
and $g_{e'\!,t}:= 0$ if no such $i$ and $k$ exist.
Then, the gain of each path $\pi$ in $\scrA$ at trial $t$ can be expressed as an additive gain of $\pi'$ in $\scrA'$:
\[
\forall t \in [1,T], \;
\forall \pi \in \scrP, \quad
\scrU(\text{out}_t(\pi), y_t)
=
\sum_{e' \in \pi'} g_{e'\!,t} \; .
\]
\end{theorem}
We can extend the algorithms and regret guarantees presented in
Section~\ref{sec:algs} to gappy count-based gains.  To the best of our
knowledge, this provides the first efficient online algorithms with
favorable regret guarantees for gappy count-based gains in full
information, semi-bandit and full bandit settings.

\section{Conclusion and Open Problems}
\label{sec:conclusions}

We presented several new algorithms for online non-additive path
learning with very favorable regret guarantees for the full
information, semi-bandit, and full bandit scenarios.\ignore{ Our main
contribution lies in introducing an intermediate automaton with
additive gains and size independent of the alphabet size associated
with the problem. This enables us to extend existing online learning
algorithms to efficiently tackle the problem of non-additive path
learning.}  We conclude with two open problems:
(1) Non-acyclic expert automata: we assumed here that the
expert automaton $\scrA$ is acyclic and the language of patterns
$\scrL = \{ \theta_1, \ldots, \theta_p \}$ is finite.  Solving the
non-additive path learning problem with cyclic expert automaton
together with (infinite) regular language $\scrL$ of patterns remains
an open problem;
(2) Incremental construction of $\scrA'$: in this work,
regardless of the data and the setting, the context-dependent
automaton $\scrA'$ is constructed in advance as a pre-processing step.
Is it possible to construct $\scrA'$ gradually as the learner goes
through trials?  Can we build $\scrA'$ incrementally in different
settings and keep it as small as possible as the algorithm is
exploring the set of paths and learning about the revealed data?

\vfill
\pagebreak


\acks{The work of MM was partly funded by NSF CCF-1535987 
and NSF IIS-1618662. Part of this work was done while MKW 
was at UC Santa Cruz, supported by NSF grant IIS-1619271.}

\bibliography{Rahmanian19}

\vfill
\pagebreak
\appendix

\section{Applications to Ensemble Structured Prediction}
\label{app:ensemble}

The algorithms discussed in Section~\ref{sec:algs} can be used for the online learning of ensembles of structured prediction experts,
and consequently, significantly improve the performance of algorithms in a number of areas including
machine translation, speech recognition, other language processing areas, optical character recognition, and computer vision.
In structured prediction problems, the output associated with a model $h$ is a structure $y$ 
that can be decomposed and represented by $\ell$ substructures $y_1, \ldots, y_\ell$. For instance, $h$ may be a machine translation system and $y_i$
a particular word.

The problem of ensemble structured prediction can be described as follows.
The learner has access to a set of $r$ experts $h_1, \ldots, h_r$ to make an ensemble prediction.
Therefore, at each round $t \in [1,T]$, the learner can use the outputs of the $r$ experts $\text{out}_t(h_1), \ldots, \text{out}_t(h_r)$.
As illustrated in Figure~\ref{fig:ensemble}(a), each expert $h_j$ consists of $\ell$ substructures $h_j = (h_{j,1}, \ldots, h_{j,\ell})$.

\begin{figure}[H]
\centering
\begin{minipage}[R]{0.45\textwidth}

\scalebox{.8}{

\begin{tikzpicture}%
  [>=stealth,
   shorten >=1pt,
   node distance=2cm,
   on grid,
   auto,
   every state/.style={draw=black!60, fill=black!5 }
  ]
  
\tikzset{cutting/.style = {->,> = latex, very thick, yellow}}
    \tikzset{rectangle/.style={draw=gray,dashed,fill=green!1,thick,inner sep=5pt}}

\node[text width=1cm] at (-1.5,0) {\large $h_1$};
\node[text width=1cm] at (-1.5,-.8) {\large $\vdots$};
\node[text width=1cm] at (3,-.8) {\large $\vdots$};
\node[text width=1cm] at (-1.5,-2) {\large $h_r$};
\node[text width=1cm] at (3,-3.2) {\large (a)};

\node[state, very thick] (0)                  {0};
\node[state, very thick] (1) [below =of 0] {0};

\node[state] (2) [right=of 0] {1};
\node[state] (3) [right=of 1] {1};

\node[state] (4) [right=of 2] {$\cdots$};
\node[state] (5) [right=of 3] {$\cdots$};

\node[state, accepting] (6) [right=of 4] {$\ell$};
\node[state, accepting] (7) [right=of 5] {$\ell$};

\path[->]
(0) 	edge node {$h_{1,1}$} (2)
(2) 	edge node {$h_{1,2}$} (4)
(4) 	edge node {$h_{1,\ell}$} (6)
(1) 	edge node {$h_{r,1}$} (3)
(3) 	edge node {$h_{r,2}$} (5)
(5) 	edge node {$h_{r,\ell}$} (7)
;

\end{tikzpicture}
}

\end{minipage}
\hfill
\begin{minipage}[L]{0.45\textwidth}
\centering
\scalebox{.8}{

\begin{tikzpicture}%
  [>=stealth,
   shorten >=1pt,
   node distance=2cm,
   on grid,
   auto,
   every state/.style={draw=black!60, fill=black!5 }
  ]
  
\tikzset{cutting/.style = {->,> = latex, very thick, yellow}}
    \tikzset{rectangle/.style={draw=gray,dashed,fill=green!1,thick,inner sep=5pt}}

\node[text width=1cm] at (-.7,-1.2) {\large $\scrA$};
\node[text width=1cm] at (2,-2.5) {\large (b)};


\node[state, very thick] (0)                  {0};
\node[state] (1) [right=of 0] {1};
\node[state] (2) [right=of 1] {$\cdots$};
\node[state, accepting] (3) [right=of 2] {$\ell$};

\path[->]
	(0) edge[bend left]	node {$h_{1,1}$} (1)
		  edge[bend right]	node [below] {$h_{r,1}$} (1)
		  edge 	node {$\ldots$} (1)
	(1) edge[bend left]	node {$h_{1,2}$} (2)
		  edge[bend right]	node [below] {$h_{r,2}$} (2)
		  edge 	node {$\ldots$} (2)
	(2) edge[bend left]	node {$h_{1,\ell}$} (3)
		  edge[bend right]	node [below] {$h_{r,\ell}$} (3)
		  edge 	node {$\ldots$} (3)
;

\end{tikzpicture}
}

\end{minipage}

\caption{
(a) the structured experts $h_1, \ldots, h_r$.
(b) the expert automaton $\scrA$ allowing all combinations.}
\label{fig:ensemble}
\end{figure}
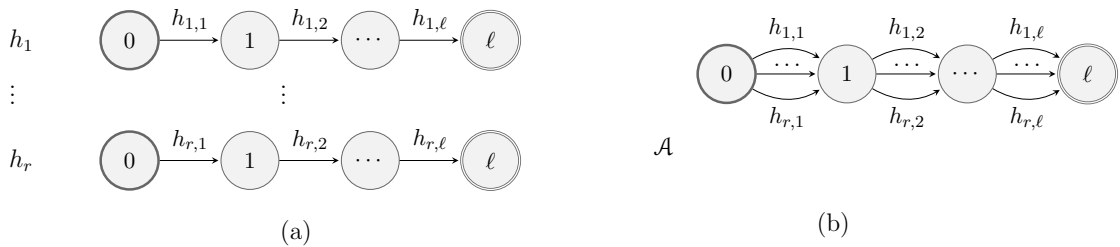

Represented by paths in an automaton, the substructures of these experts can be \emph{combined together}.
Allowing all combinations, Figure~\ref{fig:ensemble}(b) illustrates the expert automaton $\scrA$ induced by $r$ structured experts with $\ell$ substructures.
The objective of the learner is to find the best path expert which is the combination of substructures with the best expected gain.
This is motivated by the fact that one particular expert may be better at predicting one substructure while some other expert may be more
accurate at predicting another substructure. 
Therefore, it is desirable to combine the substructure predictions of all experts to obtain the more accurate prediction.

Consider the online path learning problem with expert automaton $\scrA$ in Figure~\ref{fig:ensemble}(b) 
with non-additive $n$-gram gains described in Section~\ref{sec:count-gains} for typical small values of $n$ (e.g. $n=4$). 
We construct the context-dependent automaton $\scrA'$ via a set of rewrite rules.
The rewrite rules are as follows:
$$
h_{j_1,i+1}, \; h_{j_2, i+2}, \; \ldots, h_{j_n, i+n} \; \rightarrow \#h_{j_1,i+1}h_{j_2, i+2} \ldots h_{j_n, i+n} \, / \,  \epsilon \, {\underline{\hspace{.5cm}}} \, \epsilon,
$$
for all $j_1,\ldots,j_n \in [1,r]$, $i \in [0,\ell-n]$. 
The number of rewrite rules is $(\ell - n +1) \, r^n$.
We compile these rewrite rules into the context-dependent transducer $T_\scrA$,
and then construct the context-dependent automaton $\scrA' = \Pi( \scrA \circ T_\scrA)$.

The context-dependent automaton $\scrA'$ is illustrated in Figure~\ref{fig:ensemble-a'}.
The transitions in $\scrA'$ are labeled with $n$-grams of transition names $h_{i,j}$ in $\scrA$.
The context-dependent automaton $\scrA'$ has $\ell-n+1$ layers of states each of which acts as a ``memory'' 
indicating the last observed $(n-1)$-gram of transition names $h_{i,j}$. 
With each intermediate state (i.e.~a state which is neither the initial state nor a final state), a $(n-1)$-gram is associated.
Each layer contains $r^{n-1}$ many states encoding all combinations of $(n-1)$-grams ending at that state. 
Each intermediate state has $r$ incoming transitions which are the $n$-grams ending with $(n-1)$-gram associated with the state.
Similarly each state has $r$ outgoing transitions which are the $n$-grams starting with $(n-1)$-gram associated with the state.

The number of states and transitions in $\scrA'$ are $Q=1 + r^n(\ell-n)$ and $M=r^n (\ell - n +1)$, respectively.
Note that the size of $\scrA'$ does not depend on the size of the output alphabet $\Sigma$.
Also notice that all paths in $\scrA'$ have equal length of $K = \ell - n + 1$.
Furthermore the number of paths in $\scrA'$ and $\scrA$ are the same and equal to $N = r^\ell$.

\begin{figure}
\centering
\scalebox{.8}{

\begin{tikzpicture}%
  [>=stealth,
   shorten >=1pt,
   node distance=2cm,
   on grid,
   auto,
   every state/.style={draw=black!60, fill=black!5 }
  ]
  
\tikzset{cutting/.style = {->,> = latex, very thick, yellow}}
    \tikzset{rectangle/.style={draw=gray,dashed,fill=green!1,thick,inner sep=5pt}}

\node[rectangle, anchor=south west, minimum width=1.3cm, minimum height=8cm] (A) at (1.5,-3) {};
\node[rectangle, anchor=south west, minimum width=1.3cm, minimum height=8cm] (A) at (4.5,-3) {};
\node[rectangle, anchor=south west, minimum width=1.3cm, minimum height=8cm] (A) at (6.5,-3) {};
\node[rectangle, anchor=south west, minimum width=1.3cm, minimum height=8cm] (A) at (9.5,-3) {};

\node[text width=1cm] at (0.5,-3.5) {\large $0$};
\node[text width=1cm] at (2.5,-3.5) {\large $1$};
\node[text width=1cm] at (5.5,-3.5) {\large $i-1$};
\node[text width=1cm] at (7.5,-3.5) {\large $i$};
\node[text width=3cm] at (10.5,-3.5) {\large $\ell-n+1$};

\node[text width=1cm] at (2.5,1) {\large $\vdots$};
\node[text width=1cm] at (5.5,1) {\large $\vdots$};
\node[text width=1cm] at (7.5,1) {\large $\vdots$};
\node[text width=1cm] at (10.5,1) {\large $\vdots$};

\draw [decorate,decoration={brace,amplitude=10pt},xshift=-4pt,yshift=0pt] (11.5,5) -- (11.5,-3) node [black,midway,xshift=9pt] {\large $r^{n-1}$};


\node[state, very thick] (0)                  {};

\node[state] (2) [above right=3cm of 0] {};
\node[state] (1) [above  = of 2] {};
\node[state] (3) [below right=3cm of 0] {};

\node[state] (4) [right  =3cm of 1] {};
\node[state] (5) [right  =3cm of 2] {};
\node[state] (6) [right  =3cm of 3] {};

\node[state] (7) [right  = of 4] {};
\node[state] (8) [right  = of 5] {};
\node[state] (9) [right  = of 6] {};

\node[state, accepting] (10) [right  =3cm of 7] {};
\node[state, accepting] (11) [right  =3cm of 8] {};
\node[state, accepting] (12) [right  =3cm of 9] {};

\path[->]
(0) 	edge node {$\#h_{*,1}\ldots h_{*,n}$} (2)
(5) 	edge node [below]{$\qquad\#h_{*,i}\ldots h_{*,i+n-1}$} (9)
;

\end{tikzpicture}
}
\caption{The context-dependent automaton $\scrA'$ for the expert automaton $\scrA$ depicted in Figure~\ref{fig:ensemble}(b).}
\label{fig:ensemble-a'}
\end{figure}
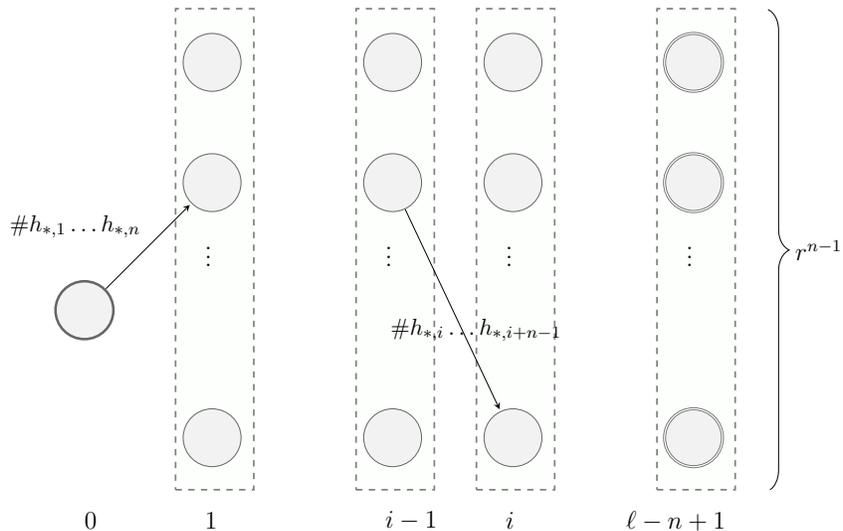

We now apply the algorithms introduced in Section~\ref{sec:algs}.

\subsection{Full Information: Context-dependent Component Hedge Algorithm}
We apply the CDCH algorithm to this application in full information setting.
The context-dependent automaton $\scrA'$ introduced in this section is highly structured.
We can exploit this structure and obtain better bounds comparing to the general bounds of Theorem \ref{th:nch} 
for CDCH.

\begin{theorem}
Let $B$ denote an upper-bound for the gains of all the transitions in $\scrA'$, and $T$ be the time horizon.
The regret of CDCH algorithm on ensemble structured prediction
with $r$ predictors consisting of $\ell$ substructures with $n$-gram gains 
can be bounded as
$$
\text{Regret}_{CDCH} \leq \sqrt{2 \, T \, B^2 \,  (\ell - n +1)^2 \, n \, \log r} + B \, (\ell - n +1) \, n \, \log r.
$$
\end{theorem}
\begin{proof}
First, note that all paths in $\scrA'$ have equal length of $K=\ell - n +1$.
Therefore there is no need of modifying $\scrA'$ to make all paths of the same length.
At each trial $t \in [T]$, we define the loss of each transition as $\ell_{t,e'}:=B- g_{t,e'}$.
Extending the results of \cite{KoolenWarmuthKivinen2010}, the general regret bound of CDCH is 
\begin{equation}
\label{eq:bounds-cdch}
\text{Regret}_{CDCH} \leq \sqrt{2 \, T \, K \, B^2 \, \Delta( \vvec_{\pi^*} || \wvec_1)} + B \, \Delta( \vvec_{\pi^*} || \wvec_1),
\end{equation}
where $\vvec_{\pi^*} \in \{0,1\}^M$ is a bit vector representation of the best comparator $\pi^*$, 
$\wvec_1 \in [0,1]^M$ is the initial weight vector in the unit-flow polytope,
and
$$
\Delta( \wvec || \wvechat) 
:=
\sum_{e' \in E'}
  \Bigg( w_{e'} \ln \frac{w_{e'}}{\widehat{w}_{e'}} + \widehat{w}_{e'} -
  w_{e'}\Bigg).
$$
Since the initial state has $r^n$ outgoing transitions, 
and all the intermediate states have $r$ incoming and outgoing transitions, 
the initial vector $\wvec_1= \frac{1}{r^n} \mathbf{1}$ falls into the unit-flow polytope,
where $\mathbf{1}$ is a vector of all ones.
Also $\vvec_{\pi^*}$ has exactly $K=\ell - n +1$ many ones.
Therefore:
\begin{equation}
\label{eq:delta-cdch}
\Delta( \vvec_{\pi^*} || \wvec_1) = (\ell - n +1) \, n \, \log r
\end{equation}

Combining the Equations (\ref{eq:bounds-cdch}) and (\ref{eq:delta-cdch}) gives us the desired regret bound.

\vspace{-.5cm}

\end{proof}

\subsection{Semi-Bandit: Context-dependent Semi-Bandit Algorithm}

In order to apply the algorithm CDSB in semi-bandit setting in this
application, we need to introduce a set of ``covering paths'' $C$ in
$\scrA'$.  We introduce $C$ by partitioning all the transitions in
$\scrA'$ into $r^n$ paths of length $\ell - n + 1$ iteratively as
follows.  At each iteration, choose an arbitrary path $\pi$ from the
initial state to a final state.  Add $\pi$ to the set $C$ and remove
all its transitions from $\scrA'$.  Notice that the number of incoming
and outgoing transitions for each intermediate state are always equal
throughout the iterations.  Also note that in each iteration, the
number of outgoing edges from the initial state decreases by one.
Therefore after $r^n$ iterations, $C$ contains a set of $r^n$ paths
that partition the set of transitions in $\scrA'$.

Furthermore, observe that the number of paths in $\scrA'$ and $\scrA$
are the same and equal to $N = r^\ell$.  The Corollary below is a
direct result of Theorem \ref{thm:cdsb} with $|C| = r^n$.

%
%
%

\begin{corollary}
For any $\delta \in (0,1)$, with proper tuning, the regret of the CDSB algorithm can be bounded, with probability $1-\delta$, as:
\begin{align*}
\text{Regret}_{CDSB}
\leq  2 \, B \, (\ell - n +1) \, \sqrt{T } \left( \sqrt{4 \, r^n \, \ell \, \ln  r}  + \sqrt{ r^n  \ln \frac{ r^n (\ell - n +1)}{\delta}} \right).
\end{align*}
\end{corollary}

\subsection{Full Bandit: Context-dependent ComBand Algorithm}
We apply the CDCB algorithm to this application in full bandit setting.
The Corollary below, which is a direct result of Theorem \ref{thm:cdcb},
give regret guarantee for CDCB algorithm.

\begin{corollary}
Let $\lambda_\text{min}$ denote the smallest non-zero eigenvalue of
$\mathbb{E}[\vvec_\pi \vvec_\pi^T]$
where $\vvec_\pi \in \{0,1\}^{M}$ is the bit representation of the path $\pi$
which is distributed according to the uniform distribution $\mu$.
With proper tuning, the regret
of CDCB can be bounded as follows:
\begin{align*}
\text{Regret}_{CDCB}
\leq 2 \, B \, \sqrt{T \, \Bigg( \frac{2(\ell - n + 1)}{r^n (\ell - n +1) \lambda_\text{min}} + 1 \Bigg) \, r^n \, (\ell - n +1) \, \ell \, \ln r  }.
\end{align*}
\end{corollary}

\section{Path Learning for Full Bandit and Arbitrary Gain}
\label{sec:exp3}

So far, we have discussed the count-based gains as a wide family of
non-additive gains in full information, semi- and full bandit
settings.  We developed learning algorithms that exploit the special
structure of count-based gains.  In this section, we go beyond the
count-based gains in the full bandit setting and consider the scenario
where the gain function is arbitrary and admits no known structure.
In other words, any two paths can have completely independent gains
regardless of the number of overlapping transitions they may share.
Clearly, the CDCB algorithm of Section~\ref{sec:algs} cannot be
applied to this case as it is specialized to (gappy) count-based
gains.  However, we present a general algorithm for path learning in
the full bandit setting, when the gain function is arbitrary.  This
algorithm (called \EXPAG) admits weaker regret bounds since no special
structure of the gains can be exploited in the absence of any
assumption.  The algorithm is essentially an efficient implementation
of \EXP\ for path learning with arbitrary gains using weighted
automata and graph operations.

We start with a brief description of the EXP3 algorithm of
\cite{auer2002nonstochastic}, which is an online learning algorithm
designed for the full bandit setting over a set of $N$ experts.  The
algorithm maintains a distribution $\wvec_t$ over the set of experts,
with $\wvec_1$ initialized to the uniform distribution.  At each round
$t \in [T]$, the algorithm samples an expert $I_t$ according to
$\wvec_t$ and receives (only) the gain $g_{t, I_t}$ associated to that
expert.  It then updates the weights multiplicatively via the rule
$w_{t+1,i} \propto w_{t,i} \, \exp ( \eta \, \gtil_{t,i})$ for all $i
\in [N]$, where 
$\gtil_{t,i} = \frac{{g_{t,i}}}{{w_{t,i}}} \mathbf{1}\{I_t = i\}$
is an unbiased surrogate gain associated with expert $i$.
The weights $w_{t+1,i}$ are then normalized to sum to
one.\footnote{ The original EXP3 algorithm of
  \cite{auer2002nonstochastic} mixes the weight vector with the
  uniform distribution in each trial.  Later
  \cite{stoltz2005information} showed that the mixing step is not
  necessary.}

In our learning scenario, each expert is a path in $\scrA$. Since the
number of paths is exponential in the size of $\scrA$,
maintaining a weight per path is computationally intractable. 
We cannot exploit the properties of the gain function
since it does not admit any known structure. However, we
can make use of the graph representation of the experts.  
We will show that the weights of
the experts at round $t$ can be compactly represented by a
deterministic \emph{weighted finite automaton} (WFA) $\scrW_t$. We
will 
further show that sampling a path from $\scrW_t$ and updating $\scrW_t$ can be done efficiently.

A deterministic WFA $\scrW$ is a deterministic finite automaton whose
transitions and final states carry weights.  Let $w(e)$ denote
the weight of a transition $e$ and $w_f(q)$ the weight at a final
state $q$.  The weight $\scrW(\pi)$ of a path $\pi$ ending in a final
state is defined as the product of its constituent transition weights
and the weight at the final state: $\scrW(\pi) := (\prod_{e \in \pi}
w(e)) \cdot w_f(\text{dest}(\pi))$, where $\text{dest}(\pi)$ denotes the
destination state of $\pi$.

Sampling paths from a deterministic WFA $\scrW$ is straightforward
when it is \emph{stochastic}, that is when the weights of all outgoing
transitions and the final state weight (if the state is final) sum to
one at every state: starting from the initial state, we can randomly
draw a transition according to the probability distribution defined by
the outgoing transition weights and proceed similarly from the
destination state of that transition, until a final state is reached.
The WFA we obtain after an update may not be stochastic, but we can
efficiently compute an equivalent stochastic WFA $\scrW'$ from any
$\scrW$ using the \emph{weight-pushing} algorithm \citep{Mohri1997,
  Mohri2009, TakimotoWarmuth2003}: $\scrW'$ admits the same states and
transitions as $\scrW$ and assigns the same weight to a path from the
initial state to a final state; but the weights along paths are
redistributed so that $\scrW'$ is stochastic. For an acyclic input WFA
such as those we are considering, the computational complexity of
weight-pushing is linear in the sum of the number of states and
transitions of $\scrW$, see the Appendix~\ref{app:fst} for details.

We now show how $\scrW_t$ can be efficiently updated using the
standard WFA operation of \emph{intersection} (or \emph{composition})
with a WFA $\scrV_t$ representing the multiplicative weights that we
will refer to as the \emph{update WFA} at time $t$.

\begin{figure}
 \centering
\begin{minipage}[R]{0.45\textwidth}
\scalebox{.65}{
\begin{tikzpicture}%
  [>=stealth,
   shorten >=1pt,
   node distance=2cm,
   on grid,
   auto,
   every state/.style={draw=black!60, fill=black!5 }
  ]
  
\tikzset{cutting/.style = {->,> = latex, very thick, yellow}}
    \tikzset{rectangle/.style={draw=gray,dashed,fill=green!1,thick,inner sep=5pt}}

\node[text width=1cm] at (-.7,-1.2) {\LARGE $ \mathcal{V}_t$};

\node[state, accepting, very thick] (0)                  {$\mathbf{0 \vert \Blue{1}}$};
\node[state, accepting] (1) [right=of 0] {$\mathbf{1 \vert \Blue{1}}$};
\node[state, accepting] (2) [right=of 1] {$\mathbf{2 \vert \Blue{1}}$};
\node[state, accepting] (99) [below = of 2] {\textbf{else}$\mathbf{ \vert \Blue{1}}$};
\node[state, accepting] (3) [right=of 2] {$\mathbf{\cdots}$};
\node[state, accepting] (4) [right=of 3] {$\mathbf{k \vert \Blue{\exp(\frac{\eta \, g_{\pi_t,t}}{\scrW_t(\pi_t)})}}$};

\path[->]
	(0) edge	node [above]  {$\mathbf{e_1 \vert \Blue{1}}$} (1)
		edge	node [below] {$\mathbf{\Green{\rho} \vert \Blue{1}}$} (99)
	(1) edge	node [below]  {$\mathbf{e_2 \vert \Blue{1}}$} (2)
		edge	node [below, near start] {$\mathbf{\Green{\rho} \vert \Blue{1}}$} (99)
	(2) edge	node [below] {$\mathbf{e_3 \vert \Blue{1}}$} (3)
		edge node [below, near start] {$\mathbf{\Green{\rho} \vert \Blue{1}}$} (99)
	(3) edge	node [below] {$\mathbf{e_k \vert \Blue{1}}$} (4)
	edge node [below, near start] {$\mathbf{\Green{\rho} \vert \Blue{1}}$} (99)
	(4) edge	node [below, near start] {$\mathbf{E \vert \Blue{1}}$} (99)
	(99) edge[loop left]	node [below] {$\mathbf{E \vert \Blue{1}}$} (99)
   ;
   

\end{tikzpicture}
}\caption{
The update WFA $\scrV_t$.
The \Blue{weight} of each state and transition is written next to its name separated 
by ``$\vert$'': $\xrightarrow{\;e \vert \Blue{\text{weight}}\;}$
}
\label{fig:updateWFA}
 \end{minipage}
 \hfill
 \centering
\begin{minipage}[R]{0.5\textwidth}
\begin{algorithmic}[1]
\State $\scrW_1 \longleftarrow \scrA \quad $
\State For $t = 1, \ldots, T$
\State \quad $\scrW_t \longleftarrow  \Call{WeightPush}{\scrW_t}$
\State \quad $\pi_t \longleftarrow  \Call{Sample}{\scrW_t}$
\State \quad $g_{\pi_t,t} \longleftarrow  \Call{ReceiveGain}{\pi_t}$
\State \quad $\scrV_t \longleftarrow  \Call{UpdateWFA}{\pi_t, \scrW_t(\pi_t), g_{t,\pi_t}}$
\State \quad $\scrW_{t+1} \longleftarrow \scrW_{t} \circ \scrV_t$
\end{algorithmic}
\caption{Algorithm \EXPAG}
\label{alg:exp3}
\end{minipage}
\end{figure}

$\scrV_t$ is a deterministic WFA that assigns weight
$\exp(\eta \, \widetilde{g}_{t, \pi})$ to path $\pi$. Thus, since
$\widetilde{g}_{t, \pi} = 0$ for all paths but the path $\pi_t$
sampled at time $t$, $\scrV_t$ assigns weight $1$ to all paths
$\pi \neq \pi_t$ and weight
$\exp\big( \frac{\eta g_{t, \pi_t}}{\scrW_t(\pi_t)} \big)$ to $\pi_t$.
$\scrV_t$ can be constructed deterministically as illustrated in
Figure~\ref{fig:updateWFA}, using \emph{$\rho$-transitions} (marked
with $\Green{\rho}$ in \Green{green}). A $\rho$-transition admits the
semantics of the \emph{rest}: it matches any symbol that is not
labeling an existing outgoing transition at that state. For example,
the $\rho$-transition at state $1$ matches any symbol other than
$e_2$.  $\rho$-transitions lead to a more compact representation not
requiring the knowledge of the full alphabet. This further helps speed
up subsequent intersection operations \citep{openfst}.

To update the weights $\scrW_t$, we use the \emph{intersection} (or
\emph{composition}) of WFAs. By definition, the intersection of
$\scrW_t$ and $\scrV_t$ is a WFA denoted by $(\scrW_t \circ \scrV_t)$
that assigns to each path expert $\pi$ the product of the weights
assigned by $\scrW_t$ and $\scrV_t$:\footnote{The terminology of
  intersection is motivated by the case where the weights are either
  $0$ or $1$, in which case the set of paths with non-zero weights in
  $\scrW_t \circ \scrV_t$ is the intersection of the sets of paths
  with with weight $1$ in $\scrW_t$ and $\scrV_t$.}
\[
  \forall \pi \in \scrP \Colon \quad (\scrW_t \circ \scrV_t)(\pi)
  = \scrW_t(\pi) \cdot \scrV_t(\pi).
\]
There exists a general an efficient algorithm for computing the
intersection of two WFAs
\citep{PereiraRiley1997,MohriPereiraRiley1996,Mohri2009}: the states
of the intersection WFA are formed by pairs of a state of the first
WFA and a state of the second WFA, and the transitions obtained by
matching pairs of transitions from the original WFAs, with their
weights multiplied, see Appendix~\ref{app:fst} for more details.
Since both $\scrW_t$ and $\scrV_t$ are deterministic, their
intersection $(\scrW_t \circ \scrV_t)$ is also deterministic
\citep{CortesKuznetsovMohriWarmuth2015}.

The following lemma (proven in Appendix~\ref{app:proofs})
shows that the weight assigned by \EXPAG\ to each
path expert coincides with those defined by \EXP.

\begin{lemma}
\label{lemma:exp3}
At each round $t \in [T]$ in \EXPAG, the following properties hold for
$\scrW_t$ and $\scrV_t$:
\begin{equation*}
\scrW_{t+1}(\pi) \propto \exp({ \eta \, \sum_{s=1}^t \widetilde{g}_{s, \pi}}), \quad
\scrV_t(\pi) = \exp(\eta \, \widetilde{g}_{t, \pi}),
\quad \text{s.t.\ } 
\widetilde{ g}_{s,\pi} = ({g_{s,\pi}}/{\scrW_s({\pi})}) \cdot \mathbf{1}\{\pi= \pi_s\}.
\end{equation*}
\end{lemma}
Figure~\ref{alg:exp3} gives the pseudocode of \EXPAG.  The time
complexity of \EXPAG\ is dominated by the cost of the intersection
operation (line $7$).  The worst-case space and time complexity of the
intersection of two deterministic WFA is linear in the size of the
automaton the algorithm returns.  Due to the specific structure of
$\scrV_t$, the size of $\scrW_{t} \circ \scrV_t$ can be shown to be at
most $O(|\scrW_{t}| + |\scrV_t|)$ where $|\scrW_t|$ is the sum of the
number of states and transitions in $\scrW_t$.  This is significantly
better than the worst case size of the intersection in general (i.e.\
$O(|\scrW_t| |\scrV_t|)$.  Recall that $\scrW_{t+1}$ is deterministic.
Thus, unlike the algorithms of \cite{CortesKuznetsovMohriWarmuth2015},
no further determinization is required.  The following Lemma
guarantees the efficiency of \EXPAG\ algorithm.  See
Appendix~\ref{app:proofs} for the proof.

\begin{lemma}
\label{lemma:exp3-time}
The time complexity of \EXPAG\ at round $t$ is in
$O(|\scrW_{t}| + |\scrV_t|)$.  Moreover, in the worst case, the growth
of $|\scrW_t|$ over time is at most linear in $K$ where $K$ is the
length of the longest path in $\scrA$.
\end{lemma}
The following upper bound holds for the regret of \EXPAG, as a direct
consequence of existing guarantees for \EXP\
\citep{auer2002nonstochastic}.

\begin{theorem}
\label{thm:exp3reg}
Let $U > 0$ be an upper bound on all path gains: $g_{t, \pi} \leq U$
for all $t \in [T]$ and all path $\pi$. Then, the regret of
\EXPAG\ with $N$ path experts is upper bounded by $U
\sqrt{2 \, T \, N \log N}$.
\end{theorem}

The $\sqrt{N}$ dependency of the bound suggests that the guarantee
will not be informative for large values of $N$.  However, the
following known lower bound shows that, in the absence of any
assumption about the structure of the gains, this dependency cannot be
improved in general \citep{auer2002nonstochastic}.

\begin{theorem}
\label{thm:exp3low}
Let $U > 0$ be an upper bound on all path gains: $g_{t, \pi} \leq U$
for all $t \in [T]$ and all path $\pi$.  Then, For any number of path
experts $N \geq 2$ there exists a distribution over the assignment of
gains to path experts such that the regret of any algorithm is at
least $\frac{1}{20} U \min \{ \sqrt{T \, N}, T\}$.
\end{theorem}

\section{Weighted Finite Automata}
\label{app:fst}

In this section, we formally describe several WFA operations relevant
to this paper, as well as their properties.

\subsection{Intersection of WFAs}

The intersection of two WFAs $\scrA_1$ and $\scrA_2$ is a WFA denoted
by $\scrA_1 \circ \scrA_2$ that accepts the set of sequences accepted
by both $\scrA_1$ and $\scrA_2$ and is defined for all $\pi$ by
\[
(\scrA_1 \circ \scrA_2)(\pi) = \scrA_1(\pi) \cdot \scrA_2(\pi).
\]
There exists a standard efficient algorithm for computing the
intersection WFA
\citep{PereiraRiley1997,MohriPereiraRiley1996,Mohri2009}.
States $Q \subseteq Q_1 \times Q_2$ of $\scrA_1 \circ \scrA_2$ are
identified with pairs of states $Q_1$ of $\scrA_1$ and $Q_2$ of
$\scrA_2$, as are the set of initial and final states.  Transitions
are obtained by matching pairs of transitions from each WFA and
multiplying their weights:
\[
\Big(q_1 \stackrel{a \vert \Blue{w_1}}{\longrightarrow} q_1', \; q_2 \stackrel{a \vert \Blue{w_2}}{\longrightarrow} q_2' \Big)
\quad \Rightarrow  \quad (q_1, q_2) \stackrel{a \vert \Blue{w_1 \cdot w_2}}{\longrightarrow} (q_1', q_2').
\]
The worst-case space and time complexity of the intersection of two
deterministic WFAs is linear in the size of the automaton the
algorithm returns.  In the worst case, this can be as large as the
product of the sizes of the WFA that are intersected (i.e.\
$O(|\scrA_1| |\scrA_2|)$.  This corresponds to the case where every
transition of $\scrA_1$ can be paired up with every transition of
$\scrA_2$.  In practice, far fewer transitions can be matched.

Notice that, when both $\scrA_1$ and $\scrA_2$ are deterministic, then
$\scrA_1 \circ \scrA_2$ is also deterministic since there is a unique
initial state (pair of initial states of each WFA) and since there is
at most one transition leaving $q_1 \in Q_1$ or $q_2 \in Q_2$ labeled
with a given symbol $a \in \Sigma$.

\subsection{Weight Pushing}

Given a WFA $\scrW$, the weight pushing algorithm
\citep{Mohri1997,Mohri2009} computes an equivalent stochastic WFA.
The weight pushing algorithm is defined as follows.  For any state $q$
in $\scrW$, let $d[q]$ denote the sum of the weights of all paths from
$q$ to final states:
\[
d[q] = \sum_{\pi \in \scrP(q)}  \Bigg(\prod_{e \in \pi} w(e)\Bigg) \cdot w_f(\text{dest}(\pi)),
\]
where $\scrP(q)$ denotes the set of paths from $q$ to final states in
$\scrW$.  The weights $d[q]$s can be computed be simultaneously for
all $q$s using standard shortest-distance algorithms over the
probability semiring \citep{Mohri2002}.  The weight pushing algorithm
performs the following steps.  For any transition $(q, q') \in E$ such
that $d[q] \neq 0$, its weight is updated as below:
\[
w(q, q') \leftarrow d[q]^{-1} \, w(q,q') \, d[q'].
\]
For any final state $q$, the weight is updated as follows:
\[
w_f(q) \leftarrow w_f(q) \, d[q]^{-1}.
\]
The resulting WFA is guaranteed to preserve the path expert weights
and to be stochastic \citep{Mohri2009}.

\section{Proofs}
\label{app:proofs}

\begin{customlemma}{\ref{lemma:exp3}}
At each round $t \in [T]$, the following properties hold for $\scrW_t$ and $\scrV_t$:
\begin{equation*}
\scrW_{t+1}(\pi) \propto \exp({ \eta \, \sum_{s=1}^t \widetilde{g}_{s, \pi}}), \quad
\scrV_t(\pi) = \exp(\eta \, \widetilde{g}_{t, \pi}),
\quad \text{s.t. } 
\widetilde{ g}_{s,\pi} = 
\begin{cases}
\frac{g_{s,\pi}}{\scrW_s({\pi})} & \pi= \pi_s \\
0 & \text{otherwise.}
\end{cases}
\end{equation*}
\end{customlemma}
\begin{proof}
  Consider $\scrV_t$ in Figure~\ref{fig:updateWFA} and let
  $\pi_t = e_1 e_2 \ldots e_k$ be the path chosen by the learner.
  Every state in $\scrV_t$ is a final state. Therefore, $\scrV_t$
  accepts any sequence of transitions names.  Moreover, since the
  weights of all transitions are $1$, the weight of any accepting path
  is simply the weight of its final state.  The construction of
  $\scrV_t$ ensures that the weight of every sequence of transition
  names is $1$, except for $\pi_t = e_1 e_2 \ldots e_k$.  Thus,
  the property of $\scrV_t$ is achieved:
$$
\scrV_t(\pi) = 
\begin{cases}
\exp\left( \frac{\eta \, g_{t,\pi}}{\scrW_t({\pi})} \right) & \pi= \pi_t \\
1 & \text{otherwise}
\end{cases}
$$

To proof of the result for $\scrW_{t+1}$ is by induction on $t$.
Consider the base case of $t = 0$.  $\scrW_1$ is initialized to the
automaton $\scrA$ with all weights being one.  Thus, the weights of
all paths are equal to $1$ before weight pushing (i.e.\
$\scrW_1(\pi) \propto 1$).  The inductive step is as follows:
\begin{align*}
\scrW_{t+1}(\pi) &\propto \scrW_t(\pi) \cdot \scrV_t(\pi) && \text{(definition of composition)} \\
&=  \exp({ \eta \, \sum_{s=1}^{t-1} \widetilde{g}_{s, \pi}})  \cdot  \exp(\eta \, \widetilde{g}_{t, \pi})  && \text{(induction hypothesis)}\\
&=  \exp({ \eta \, \sum_{s=1}^{t} \widetilde{g}_{s, \pi}}),
\end{align*}
which completes the proof.
\end{proof}

\begin{customlemma}{\ref{lemma:exp3-time}}
The time complexity of \EXPAG\ at round $t$ is in $O(|\scrW_{t}| + |\scrV_t|)$. 
Moreover, in the worst case, the growth of $|\scrW_t|$ over time is at most linear in $K$
where $K$ is the length of the longest path in $\scrA$.
\end{customlemma}
\begin{proof}
  Figure~\ref{alg:exp3} gives the pseudocode of \EXPAG.  The time
  complexity of the weight-pushing step is in $O(|\scrW_t|)$, where
  $|\scrW_t|$ is the sum of the number of states and transitions in
  $\scrW_t$.  Lines $4$ and $6$ in Algorithm~\ref{alg:exp3} take
  $O(|\scrV_t|)$ time.  Finally, regarding line $7$, the worst-case
  space and time complexity of the intersection of two deterministic
  WFA is linear in the size of the automaton the algorithm returns.
  However, the size of the intersection automaton
  $\scrW_t \circ \scrV_t$ is significantly smaller than the general
  worst case (i.e.\ $O(|\scrW_t| |\scrV_t|)$) due to the state
  ``else'' with all in-coming $\rho$-transitions (see
  Figure~\ref{fig:updateWFA}).  Since $\scrW_t$ is deterministic, in
  the construction of $\scrW_t \circ \scrV_t$, each state of $\scrV_t$
  except from the ``else'' state is paired up only with one state of
  $\scrW_t$.  For example, if the state is the one reached by
  $e_1 e_2 e_3$, then it is paired up with the single state of
  $\scrW_t$ reached when reading $e_1 e_2 e_3$ from the initial state.
  Thus $| \scrW_t \circ \scrV_t | \leq |\scrW_{t}| + |\scrV_t|$, and
  therefore, the intersection operation in line $7$ takes
  $O(|\scrW_{t}| + |\scrV_t|)$ time, which also dominates the time
  complexity of \EXPAG\ algorithm.

  Additionally, observe that the size $|\scrV_t|$ is in $O(K)$ where
  $K$ is the length of the longest path in $\scrA$.  Since
  $ |\scrW_{t+1}| = | \scrW_t \circ \scrV_t | \leq |\scrW_{t}| +
  |\scrV_t| $, in the worst case, the growth of $|\scrW_t|$ over time
  is at most linear in $K$.
\end{proof}

\begin{customproposition}{\ref{prop:1-1}}
  Let $\scrA$ be an expert automaton and let $T_\scrA$ be a
  deterministic transducer representing the rewrite
  rules~\eqref{eq:rules}.  Then, for each accepting path $\pi$ in
  $\scrA$ there exists a unique corresponding accepting path $\pi'$ in
  $\scrA' = \Pi ( \scrA \circ T_\scrA )$.
\end{customproposition}

\begin{proof}
  To establish the correspondence, we introduce $T_\scrA$ as a mapping
  from the accepting paths in $\scrA$ to the accepting paths in
  $\scrA'$.  Since $T_\scrA$ is deterministic, for each accepting path
  $\pi$ in $\scrA$ (i.e.\ $\scrA(\pi)=1$), $T_\scrA$ assigns a unique
  output $\pi'$, that is $T_\scrA(\pi, \pi')=1$.  We show that $\pi'$
  is an accepting path in $\scrA'$. Observe that
$$
(\scrA \circ T_\scrA)(\pi, \pi') = \scrA(\pi) \cdot T_\scrA(\pi, \pi') = 1 \times 1 = 1,
$$
which implies that
$\scrA'(\pi') = \Pi ( \scrA \circ T_\scrA )(\pi') = 1$.  Thus for each
accepting path $\pi$ in $\scrA$ there is a unique accepting path
$\pi'$ in $\scrA'$.
\end{proof}

\begin{customthm}{\ref{thm:additivity}}
Given the trial $t$, for each transition $e' \in E'$ in $\scrA'$
define the gain $g_{e'\!,t} := \left[ \Theta(y_t) \right]_i$ if
$\out_t(e') = \theta_i$ for some $i$ and $g_{e'\!,t}:=0$ if no
such $i$ exists.  Then, the gain of each path $\pi$ in $\scrA$ at
trial $t$ can be expressed as an additive gain of $\pi'$ in $\scrA'$:
\begin{equation*}
\forall t \in [1,T], \;
\forall \pi \in \scrP \Colon \quad
\scrU(\out_t(\pi), y_t)
= \sum_{e' \in \pi'} g_{e'\!,t} \;.
\end{equation*}
\end{customthm}

\begin{proof}
By definition, the $i$th component of $\Theta(\text{out}_t(\pi))$ is the number of occurrences of $\theta_i$ in $\text{out}_t(\pi)$.
Also, by construction of the context-dependent automaton based on rewrite rules, 
$\pi'$ contains all path segments of length $|\theta_i|$ of $\pi$ in $\scrA$ as transition labels in $\scrA'$.
Thus every occurrence of $\theta_i$ in $\text{out}_t(\pi)$ will appear as 
a transition label in $\text{out}_t(\pi')$.
Therefore the number of occurrences of $\theta_i$ in $\text{out}_t(\pi)$ is
\begin{equation}
\label{eq:theta_i}
\left[\Theta(\text{out}_t(\pi))\right]_i
= \sum_{e' \in \pi'} \mathbf{1}\{ \text{out}_t(e')=\theta_i \},
\end{equation}
where $\mathbf{1}\{\cdot\}$ is the indicator function. 
Thus, we have that
\begin{align*}
\scrU(\text{out}_t(\pi), y_t)
&= \Theta(y_t) \cdot \Theta(\text{out}_t(\pi)) && (\text{definition of } \scrU)  \\
&= \sum_i \left[\Theta(y_t)\right]_i \left[\Theta(\text{out}_t(\pi))\right]_i \\
&= \sum_i \left[\Theta(y_t)\right]_i \sum_{e' \in \pi'} \mathbf{1}\{ \text{out}_t(e')=\theta_i \} && (\text{Equation (\ref{eq:theta_i})})\\
&= \sum_{e' \in \pi'} \underbrace{\sum_i \left[\Theta(y_t)\right]_i  \mathbf{1}\{ \text{out}_t(e')=\theta_i \}}_{=g_{e'\!,t}},
\end{align*}
which concludes the proof.
\end{proof}

\begin{theorem}
\label{thm:gap-additivity}
Given the trial $t$ and discount rate $\gamma \in [0,1]$, for each transition $e' \in E'$ in $\scrA'$ define the gain $g_{e'\!,t} := \gamma^k \left[ \Theta(y_t) \right]_i$ 
if $\text{out}_t(e') = (\theta_i)_k$ for some $i$ and $k$ 
and $g_{e'\!,t}:=0$ if no such $i$ and $k$ exist.
Then, the gain of each path $\pi$ in $\scrA$ at trial $t$ can be expressed as an additive gain of $\pi'$ in $\scrA'$:
$$
\forall t \in [1,T], \;
\forall \pi \in \scrP \Colon \quad
\scrU(\text{out}_t(\pi), y_t)
=
\sum_{e' \in \pi'} g_{e'\!,t} \; .
$$
\end{theorem}

\begin{proof}
By definition, the $i$th component of $\Theta(\text{out}_t(\pi))$ is the discounted count of gappy occurrences of $\theta_i$ in $\text{out}_t(\pi)$.
Also, by construction of the context-dependent automaton based on rewrite rules, 
$\pi'$ contains all gappy path segments of length $|\theta_i|$ of $\pi$ in $\scrA$ as transition labels in $\scrA'$.
Thus every gappy occurrence of $\theta_i$ with $k$ gaps in $\text{out}_t(\pi)$ will appear as 
a transition label $(\theta_i)_k$ in $\text{out}_t(\pi')$.
Therefore the discounted counts of gappy occurrences of $\theta_i$ in $\text{out}_t(\pi)$ is
\begin{equation}
\label{eq:gap-theta_i}
\left[\Theta(\text{out}_t(\pi))\right]_i
= \sum_{e' \in \pi'} \sum_k \gamma^k \, \mathbf{1}\{ \text{out}_t(e')=(\theta_i)_k \}.
\end{equation}
Therefore, the following holds:
\begin{align*}
\scrU(\text{out}_t(\pi), y_t)
&= \Theta(y_t) \cdot \Theta(\text{out}_t(\pi)) && (\text{definition of } \scrU)  \\
&= \sum_i \left[\Theta(y_t)\right]_i \left[\Theta(\text{out}_t(\pi))\right]_i \\
&= \sum_i \left[\Theta(y_t)\right]_i \sum_{e' \in \pi'} \sum_k \gamma^k \, \mathbf{1}\{ \text{out}_t(e')=(\theta_i)_k \} && (\text{Equation (\ref{eq:gap-theta_i})})\\
&= \sum_{e' \in \pi'} \underbrace{\sum_i \sum_k \gamma^k \, \left[\Theta(y_t)\right]_i  \mathbf{1}\{ \text{out}_t(e')=(\theta_i)_k \}}_{=g_{e'\!,t}},
\end{align*}
and the proof is complete.
\end{proof}

\section{Gains $\scrU$ vs Losses $-\log(\scrU)$}
\label{app:log-reg}

Let $\scrU$ be a non-negative gain function. 
Also let $\pi^* \in \scrP$ be the best comparator over the $T$ rounds.
The regret associated with $\scrU$ and $-\log \scrU$,
which are respectively denoted by $R_G$ and $R_L$,
are defined as below:
\begin{align*}
R_G &:= \sum_{t=1}^T  \scrU(\text{out}_t(\pi^\ast), y_t) - \scrU(\text{out}_t(\pi_t), y_t),  \\
R_L &:= \sum_{t=1}^T -\log(\scrU(\text{out}_t(\pi_t), y_t)) - (-\log(\scrU(\text{out}_t(\pi^\ast), y_t))).
\end{align*}
Observe that if $\scrU(\text{out}_t(\pi_t), y_t) = 0$ for any $t$,
then $R_L$ is unbounded.  Otherwise, let us assume that there exists a
positive constant $\alpha > 0$ such that
$\scrU(\text{out}_t(\pi_t), y_t) \geq \alpha$ for all $t \in [T]$.
Note, for count-based gains, we have $\alpha \geq 1$, since all
components of the representation $\Theta(\cdot)$ are non-negative
integers.  Thus, the next proposition shows that for count-based gains
we have $R_L \leq R_G$.

\begin{proposition}
  Let $\scrU$ be a non-negative gain function.  Assume that there
  exists $\alpha > 0$ such that
  $\scrU(\text{out}_t(\pi_t), y_t) \geq \alpha$ for all $t \in [T]$.
  Then, the following inequality holds:
  $R_L \leq \frac{1}{\alpha} R_G$.
\end{proposition}

\begin{proof}
The following chain of inequalities hold:
\begin{align*}
R_L 
& = \sum_{t = 1}^T -\log(\scrU(\text{out}_t(\pi_t), y_t)) - (-\log(\scrU(\text{out}_t(\pi^\ast), y_t))) \\
& = \sum_{t = 1}^T \log \bigg[ \frac{\scrU(\text{out}_t(\pi^\ast),
  y_t)}{\scrU(\text{out}_t(\pi_t), y_t)} \bigg] \\
& = \sum_{t = 1}^T \log \bigg[ 1 + \frac{\scrU(\text{out}_t(\pi^\ast),
  y_t) - \scrU(\text{out}_t(\pi_t), y_t)}{\scrU(\text{out}_t(\pi_t),
  y_t)} \bigg]\\
& \leq \sum_{t = 1}^T  \frac{\scrU(\text{out}_t(\pi^\ast), y_t) - \scrU(\text{out}_t(\pi_t), y_t)}{\scrU(\text{out}_t(\pi_t), y_t)} && (\text{since }\log(1+x) \leq x)\\
& \leq \frac{1}{\alpha} \sum_{t = 1}^T  \scrU(\text{out}_t(\pi^\ast), y_t) - \scrU(\text{out}_t(\pi_t), y_t) && (\text{since } \scrU(\text{out}_t(\pi_t),y_t) \geq \alpha) \\
& \leq \frac{1}{\alpha} R_G,
\end{align*}
which completes the proof.
\end{proof}

\section{Non-Additivity of the Count-Based Gains }
\label{app:non-additive-example}
Here we show that the count-based gains are not additive in general.
Consider the special case of $4$-gram-based gains in Figure~\ref{fig:no-add-ex}. 
Suppose the count-based gain defined in Equation~(\ref{eq:gain}) 
can be expressed additively along the transitions (proof by contradiction).
Let the target sequence $y$ be as below:
$$
y = \textbf{He would like to eat cake}
$$

Then the only $4$-gram in Figure~\ref{fig:no-add-ex} with 
positive gain is ``\textbf{He-would-like-to}''. 
Thus, if a path contains this $4$-gram, it will have a gain of $1$.
Otherwise, its gain will be $0$.
Suppose that the transitions are labeled as depicted in Figure~\ref{fig:no-add-ex-param}
and each transition $e \in E$ carries an additive gain of $g(e)$.
Consider the following four paths:
\begin{align*}
&\pi_1 = e_1 e_2 e_3 e_4 e_5 e_6 \\
&\pi_2 = e_1' e_2 e_3 e_4 e_5 e_6 \\
&\pi_3 = e_1 e_2 e_3' e_4 e_5 e_6 \\
&\pi_4 = e_1' e_2 e_3' e_4 e_5 e_6 
\end{align*}

Due to the additivity of gains, we can obtain:
\begin{align*}
\scrU(\text{out}_t(\pi_1),y) + \scrU(\text{out}_t(\pi_4),y)
&=  g(e_1) +g(e_1')+g(e_3) +g(e_3') \\ 
& \quad + 2g(e_2) + 2g(e_4) +2g(e_5) +2g(e_6)  \\
 &= \scrU(\text{out}_t(\pi_2),y) + \scrU(\text{out}_t(\pi_3),y)
\end{align*}

This, however, contradicts the definition of the count-based gains in Equation~(\ref{eq:gain}):
$$
\underbrace{\scrU(\text{out}_t(\pi_1),y)}_{=1} + \underbrace{\scrU(\text{out}_t(\pi_4),y)}_{=0}
\not =
\underbrace{\scrU(\text{out}_t(\pi_2),y)}_{=0} + \underbrace{\scrU(\text{out}_t(\pi_3),y)}_{=0}
$$
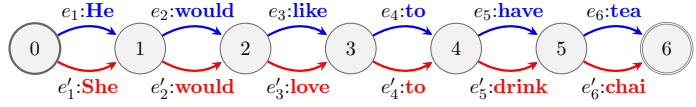
\begin{figure}
\centering
\scalebox{.7}{

\begin{tikzpicture}%
  [>=stealth,
   shorten >=1pt,
   node distance=2cm,
   on grid,
   auto,
   every state/.style={draw=black!60, fill=black!5 }
  ]
  
\tikzset{alg1/.style = {->,> = stealth, very thick, blue}}
\tikzset{alg2/.style = {->,> = stealth, very thick, red}}



\node[state, very thick] (0)                  {0};
\node[state] (1) [right=of 0] {1};
\node[state] (2) [right=of 1] {2};
\node[state] (3) [right=of 2] {3};
\node[state] (4) [right=of 3] {4};
\node[state] (5) [right=of 4] {5};
\node[state, accepting] (6) [right=of 5] {6};

\path[->]
	(0) edge[bend left]	node {$e_1$:\Blue{\textbf{He}}} (1)
		  edge[bend right]	node [below] {$e_1'$:\Red{\textbf{She}}} (1)
	(1) edge[bend left]	node {$e_2$:\Blue{\textbf{would}}} (2)
		  edge[bend right]	node [below] {$e_2'$:\Red{\textbf{would}}} (2)
	(2) edge[bend left]	node {$e_3$:\Blue{\textbf{like}}} (3)
		  edge[bend right]	node [below] {$e_3'$:\Red{\textbf{love}}} (3)
	(3) edge[bend left]	node {$e_4$:\Blue{\textbf{to}}} (4)
		  edge[bend right]	node [below] {$e_4'$:\Red{\textbf{to}}} (4)
	(4) edge[bend left]	node {$e_5$:\Blue{\textbf{have}}} (5)
		  edge[bend right]	node [below] {$e_5'$:\Red{\textbf{drink}}} (5)
	(5) edge[bend left]	node {$e_6$:\Blue{\textbf{tea}}} (6)
		  edge[bend right]	node [below] {$e_6'$:\Red{\textbf{chai}}} (6)
;

\draw[alg1] (0) [bend left] to (1);
\draw[alg1] (1) [bend left] to (2);
\draw[alg1] (2) [bend left] to (3);
\draw[alg1] (3) [bend left] to (4);
\draw[alg1] (4) [bend left] to (5);
\draw[alg1] (5) [bend left] to (6);

\draw[alg2] (0) [bend right] to (1);
\draw[alg2] (1) [bend right] to (2);
\draw[alg2] (2) [bend right] to (3);
\draw[alg2] (3) [bend right] to (4);
\draw[alg2] (4) [bend right] to (5);
\draw[alg2] (5) [bend right] to (6);

\end{tikzpicture}
}
\caption{Additive gains carried by the transitions.}
\label{fig:no-add-ex-param}
\end{figure}

\end{document}